\documentclass{article}
\usepackage[utf8]{inputenc}
\usepackage[margin=1in]{geometry}
\usepackage[numbers,sort]{natbib}

\usepackage{macros}

\title{Wide and Deep Neural Networks \\Achieve Optimality for Classification}
\author{Adityanarayanan Radhakrishnan\thanks{Laboratory for Information \& Decision Systems, and 
 Institute for Data, Systems, and Society, 
 Massachusetts Institute of Technology} $~~$  Mikhail Belkin \thanks{
Halıcıoğlu Data Science Institute, University of California, San Diego} $~~$ Caroline Uhler\textsuperscript{\specificthanks{1},}\thanks{Broad Institute of MIT and Harvard}  }
\date{\today}

\begin{document}

\maketitle

\begin{abstract}
While neural networks are used for classification tasks across domains, a long-standing open problem in machine learning is determining whether neural networks trained using standard procedures are optimal for classification, i.e., whether such models minimize the probability of misclassification for arbitrary data distributions.  In this work, we identify and construct an explicit set of neural network classifiers that achieve optimality. Since effective neural networks in practice are typically both wide and deep, we analyze infinitely wide networks that are also infinitely deep. In particular, using the recent connection between infinitely wide neural networks and Neural Tangent Kernels, we provide explicit activation functions that can be used to construct networks that achieve optimality. Interestingly, these activation functions are simple and easy to implement, yet differ from commonly used activations such as ReLU or sigmoid.  More generally, we create a taxonomy of infinitely wide and deep networks and show that these models implement one of three well-known classifiers depending on the activation function used: (1) 1-nearest neighbor (model predictions are given by the label of the nearest training example); (2) majority vote (model predictions are given by the label of the class with greatest representation in the training set); or (3) singular kernel classifiers (a set of classifiers containing those that achieve optimality).  Our results highlight the benefit of using deep networks for classification tasks, in contrast to regression tasks, where excessive depth is harmful. 
\end{abstract}



\section{Introduction}

Deep learning has produced state-of-the-art results across several application domains including computer vision~\cite{ResNet}, natural language processing~\cite{GPT3}, and biology~\cite{Alphafold}.  Despite these empirical successes, our understanding of basic theoretical properties of deep networks is far from satisfactory.  In fact,  for the fundamental problem of classification it has not been established  whether neural networks trained with standard optimization methods can achieve optimality, i.e., whether they minimize the probability of misclassification for arbitrary data distributions (a property referred to as \emph{Bayes optimality} or \emph{consistency} in the statistics literature).

There is a vast literature on the optimality of statistical machine learning methods; in particular,  given the modern practice of using models that can interpolate (i.e., fit the training data exactly), recent works analyzed the optimality of interpolating machine learning models including weighted nearest neighbor methods and kernel smoothers (also known as Nadaraya-Watson estimators)~\cite{CoverHart1NN, BelkinKNN, RakhlinInterpolationStatisticalOptimality, RakhlinLaplace, DevroyeHilbertKernel}. However, little is known about deep neural networks.  Classical work~\cite{LugosiConsistency} analyzing the optimality of neural networks utilizes the results of Cybenko~\cite{CybenkoApproximation} and Hornik~\cite{HornikApproximation} to show that the optimal classifier can be approximated by a neural network that is sufficiently wide; i.e., these prior results are concerned with the existence of networks that achieve optimality and do not present computationally feasible algorithms for finding such networks. 

By establishing a connection between interpolating kernel smoothers and deep neural networks, we identify and construct an explicit class of neural networks that, when trained with gradient descent, achieve optimality for classification problems.  Our results utilize the recent Neural Tangent Kernel (NTK) connection between training wide neural networks and using kernel methods.  Several works~\cite{NTKJacot, PenningtonNeuralNetLinearModels, BelkinNTKLinear, BelkinLossLandscapesPL} established conditions under which using a kernel method with the NTK is equivalent to training neural networks, as network width approaches infinity.  Given the conceptual simplicity of kernel methods, the NTK has been widely used as a tool for understanding the theoretical properties of neural networks~\cite{ResNetNTK, PenningtonNeuralNetLinearModels, BelkinLossLandscapesPL,  SimpleFastFlexibleMatrixCompletion, PenningtonEdgeofChaos}. Since neural networks in practice are often both wide and deep, we consider the natural extension of 
networks that are both infinitely wide and deep.  

In particular, we focus on infinitely wide and deep networks in the classification setting and show that they have markedly different behavior than in the regression setting.  Indeed, prior work~\cite{JudithNTK, ResNetNTK} showed that in the regression setting, infinitely wide and deep neural networks simply predict near-zero values at all test samples and thus, are far from optimal (see Fig.~\ref{fig: Overview Schematic}b).  As a consequence, these models were dismissed as an approach for explaining the strong performance of deep networks in practice. 
 In stark contrast to regression,  we show that the sign of the predictor can be informative even when the its numerical output is arbitrarily close to zero (see Fig.~\ref{fig: Overview Schematic}b for an illustration).  In fact, as we show in this work, this is exactly how infinitely wide and deep neural networks can achieve optimal classification accuracy even though the output of the network approaches zero.

\begin{figure*}[!t]
    \centering
    \includegraphics[width=\textwidth]{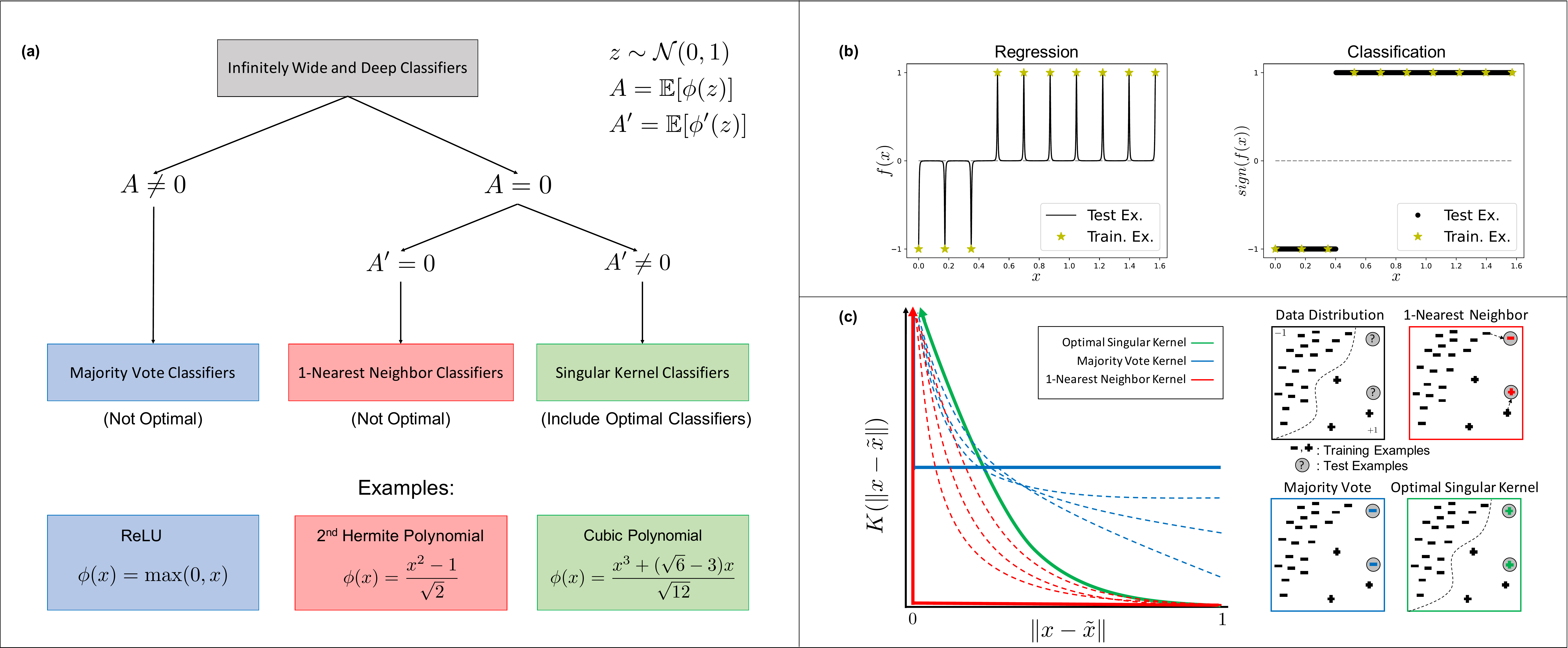}
    \caption{Behavior of infinitely wide and deep neural networks trained with gradient descent.  (a) Taxonomy of infinitely wide and deep networks.  Depending on the choice of the activation function, $\phi(\cdot)$, these models implement majority vote (blue), 1-nearest neighbor (red), or singular kernel classifiers (green), a subset of which achieve optimality. (b) Regression versus classification using infinitely wide and deep networks. While these models are not effective in the regression setting, since their predictions are near zero almost everywhere, they can achieve optimality for classification, where only the sign of the prediction matters. (c) Illustration of the different behaviors of infinitely wide and deep networks for varying activation functions.  Depending on the activation function, infinitely wide and deep networks implement majority vote (blue), 1-nearest neighbor (red), or singular kernel classifiers that can achieve optimality (green). Singular kernels that grow too slowly are akin to majority vote classifiers (dashed blue), whereas those that grow too quickly are akin to weighted nearest neighbor classifiers (dashed red).}
    \label{fig: Overview Schematic}
\end{figure*}

To characterize the behavior of infinitely wide and deep classifiers, we establish a taxonomy of such models, and we prove that it includes networks that achieve optimality (see Fig.~\ref{fig: Overview Schematic}a).  More precisely, we prove that infinitely wide and deep neural network classifiers implement one of the following three well-known classifiers depending on the choice of activation function:
\begin{enumerate}
    \item \emph{1-nearest neighbor (1-NN) classifiers}: the prediction on a new sample is the label of the nearest sample (under Euclidean distance) in the training set. 
    \item \emph{Majority vote classifiers:} the prediction on a new sample is the label of the class with greater representation in the training set. 
    \item \emph{Singular kernel classifiers:} the prediction on a new sample is obtained by using the kernel $K(x, \tilde{x}) = \frac{R(\|x - \tilde{x}\|)}{\|x - \tilde{x}\|^{\alpha}}$ where $\alpha > 0$ is the order of the singularity.\footnote{For this order to be well-defined, $R(\cdot)$ is non-negative and satisfies $\inf\limits_{|u| < \epsilon} R(u) > 0$ and $ |R(u)| < C$ for some $\epsilon, C > 0$.}  As is standard when using kernel smoothers for classification, the prediction, $m(x)$, on a new sample $x$ given training data $\{(x^{(i)}, y^{(i)})\}_{i=1}^{n}$ is
    \begin{align}
    \label{eq: Kernel Smoother}
        m(x) = \sign\Big(\sum_{i=1}^{n} y^{(i)} K(x^{(i)}, x) \Big). 
    \end{align}
\end{enumerate}


As a corollary of a result in~\cite{DevroyeHilbertKernel} it follows that singular kernel classifiers achieve optimality when $\alpha$ is the dimension of the data, $d$ (see Supplementary Information \ref{appendix: C}). Hence our taxonomy and in particular Theorem \ref{theorem: Optimality of the NTK} of this work provide exact conditions when infinitely wide and deep neural network classifiers achieve optimality for any given data dimension.  Notably, we identify a simple class of activation functions that yield singular kernel classifiers with $\alpha = d$, and we thus identify concrete examples of neural networks that achieve optimality. For example, for $d = 2$, the infinitely wide and deep classifier with activation function $\phi(x) = (x^3 + (\sqrt{6} - 3)x)/\sqrt{12}$ achieves optimality. Interestingly, the popular rectified linear unit (ReLU) activation $\phi(x) = \max(x, 0)$ leads to an infinitely wide and deep classifier that implements the majority vote classifier and is thus not optimal. Similarly, the activation function $\phi(x) = (x^2 - 1)/\sqrt{2}$ leads to an infinitely wide and deep classifier that implements the 1-NN classifier and is thus also not optimal. 

We note that singular kernels provide a natural transition between 1-NN and majority vote classifiers.  Namely, as discussed in~\cite{DevroyeHilbertKernel}, for $\alpha > d$, singular kernel classifiers behave akin to weighted nearest neighbor classifiers since $\|x - \tilde{x}\|^{\alpha}$ is extremely small for $\tilde{x}$ near $x$.  Similarly, for $\alpha < d$, singular kernel classifiers behave akin to majority vote classifiers since $\|x - \tilde{x}\|^{\alpha}$ is no longer small for $\tilde{x}$ far from $x$.  We visualize this transition between the three classes established in our taxonomy in Fig.~\ref{fig: Overview Schematic}c.  




\section{Taxonomy of Infinitely Wide and Deep Neural Networks}
\label{sec: Taxonomy of Inf Networks}

In the following, we construct a taxonomy of classifiers implemented by infinitely wide and deep neural networks. Our construction relies on the recent connection between infinitely wide neural networks and kernel methods~\cite{NTKJacot}.  In particular, this connection involves utilizing a kernel method known as a kernel machine, which is related to the kernel smoother described in Eq.~\eqref{eq: Kernel Smoother}. In contrast to the kernel smoother, a kernel machine with kernel $K$
is given by: 
\begin{align}
\label{eq: Kernel Machine with NTK}
\sign\left(y (K_n)^{-1} K(X, x)\right), 
\end{align}
where $X = [ x^{(1)} | x^{(2)} | \ldots | x^{(n)}] \in \mathbb{R}^{d \times n}$ denotes the training data, $y = [y^{(1)}, y^{(2)}, \ldots y^{(n)}] \in \{-1, 1\}^{1 \times n}$ the labels, $K_n \in \mathbb{R}^{n \times n}$ satisfies $(K_n)_{i,j} = K(x^{(i)}, x^{(j)})$ and $K(X, x) \in \mathbb{R}^{n}$ satisfies $\left(K(X, x)\right)_{i} = K(x^{(i)}, x)$. Both kernel methods can be used as prediction schemes for classification~\cite{KernelsBook}. Note that while both algorithms produce predictors with  the same functional form, their predictions are generally different. 
Indeed, understanding the relation between kernel smoothers and kernel machines will be critical to our proof of optimality.  

Under certain conditions, training a neural network as width approaches infinity is equivalent\footnote{This equivalence requires a particular initialization scheme on the weights known as the NTK initialization scheme \cite{NTKJacot}.  Formally, this equivalence holds when an offset term corresponding to the predictions of the neural network at initialization are added to those given by the using a kernel machine with the NTK~\cite{NTKJacot}. Like in prior works (e.g.~\cite{JudithNTK, CNTKArora, FiniteVsInfiniteNeuralNetworks, SimpleFastFlexibleMatrixCompletion, ResNetNTK}), we will analyze the NTK without such offset. This model corresponds to averaging the predictions of infinitely many infinite width neural networks~\cite{NeuralTangentsGoogle}.} to using a kernel machine with a specific kernel known as the Neural Tangent Kernel~\cite{NTKJacot}, which is defined below.



\begin{definition}
Let $f^{(L)}(x ; \mathbf{W})$ denote a fully connected network\footnote{Throughout this work, we consider fully connected networks that have no bias terms.} with $L$ hidden layers with parameters $\mathbf{W}$ operating on data $x \in \mathbb{R}^{d}$.  For $x, \tilde{x} \in \mathbb{R}^{d}$, the \textbf{Neural Tangent Kernel} (\textbf{NTK}) is given by: 
\begin{align*}
    K^{(L)}(x, \tilde{x}) = \langle \nabla_{\mathbf{W}} f^{(L)}(x ; \mathbf{W}), \nabla_{\mathbf{W}} f^{(L)}(\tilde{x} ; \mathbf{W}) \rangle~.
\end{align*}
\end{definition}

To work with a simple closed form for the NTK and to avoid symmetries arising from the activation function, we will consider training data with density on $\mathcal{S}_+^{d}$, where $\mathcal{S}_+^{d}$ is the intersection of the unit sphere $\mathcal{S}^{d}$ in $d+1$ dimensions and the non-negative orthant.\footnote{For example, min-max scaling followed by projection onto the sphere results in the data lying in this region.}

In this work, we analyze the behavior of infinitely wide and deep networks by analyzing the kernel machine in Eq.~\eqref{eq: Kernel Machine with NTK}, as depth, $L$, goes to infinity.  To perform our analysis, we utilize the recursive formula for the NTK of a deep network originally presented in~\cite{NTKJacot}. Namely, $K^{(L)}$ can be expressed as a function of $K^{(L-1)}$ and the network activation function, $\phi(\cdot)$, yielding a discrete dynamical system indexed by $L$.  The exact formula can be found in Eq.~\eqref{eq: NTK Recursive Formula}, and additional relevant results from prior works that are used in our proofs are referenced in Supplementary Information \ref{appendix: A}.

Remarkably, the properties of the resulting dynamical system as $L \to \infty$ are governed by the mean of $\phi(z)$ and its derivative, $\phi'(z)$, for $z \sim \mathcal{N}(0, 1)$. For simplicity, we will assume throughout that $\mathbb{E}[\phi(z)^2]<\infty$ and similarly $\mathbb{E}[\phi'(z)^2] < \infty$, an assumption that holds for many activation functions used in practice including ReLU, leaky ReLU, sigmoid, sinusoids, and polynomials. By defining $A = \mathbb{E}[\phi(z)]$ and $A' = \mathbb{E}[\phi'(z)]$, we break down our analysis into the following three cases: 
\begin{align*}
    &\text{Case 1:} ~~ A = 0  ~,~ A' \neq 0 ~,\\ 
    &\text{Case 2:} ~~ A = 0  ~,~ A' = 0 ~,\\ 
    &\text{Case 3:} ~~ A \neq 0  ~. 
\end{align*}

Under cases 1 and 2, $0$ is the unique fixed point attractor of the recurrence for $K^{(L)}$ and thus $K^{(L)}(x, \tilde{x}) \to 0$ as $L \to \infty$ for $x \neq \tilde{x}$. As a consequence, cases 1 and 2 lead to infinitely wide and deep neural networks that predict 0 almost everywhere. Thus, these networks are far from optimal in the regression setting and were thus dismissed as an approach for explaining the strong performance of deep networks.  On the other hand, case 3 yields nonzero values for any pair of examples and thus, prior works that analyzed the regression setting~\cite{JudithNTK, ResNetNTK} focused on activation functions satisfying case 3.

In stark contrast to the regression setting, we will show that infinitely wide and deep networks with activation functions satisfying case 1 are effective for classification, with a subset achieving optimality.  In particular, we will show that networks in case 1 implement singular kernel classifiers while those in case 2 implement 1-NN classifiers. Notably, we will identify conditions and provide explicit examples of activation functions in case 1 that guarantee optimality.  We will then show that infinitely wide and deep classifiers with activations satisfying case 3 generally correspond to majority vote classifiers.  A summary of our taxonomy is presented in Fig.~\ref{fig: Overview Schematic}a, and we will now discuss each of the three cases in more depth.

\subsection*{Case 1 ($A = 0, A' \neq 0$) networks implement singular kernel classifiers and can achieve optimality.}
\label{sec: Optimal NTK}


We establish conditions on the activation function under which an infinitely wide and deep network implements a singular kernel classifier (Theorem~\ref{theorem: NTK singular kernel}).  We then utilize results of~\cite{DevroyeHilbertKernel} to show that this set of classifiers contains those that achieve optimality for any given data dimension.  Lastly, we will present explicit activation functions that lead to infinitely wide and deep classifiers that achieve optimality. We begin with the following theorem, which establishes conditions under which the infinite depth limit of the NTK is a singular kernel.

\begin{theorem}
\label{theorem: NTK singular kernel}
Let $K^{(L)}$ denote the NTK of a fully connected neural network with $L$ hidden layers and activation function $\phi(\cdot)$.  For $z \sim \mathcal{N}(0, 1)$, define $A = \mathbb{E}[\phi(z)]$, $A' = \mathbb{E}[\phi'(z)]$, and $B' = \mathbb{E}[\phi'(z)^2]$.  If $A = 0$ and $A' \neq 0$, then for $x, \tilde{x} \in \mathcal{S}_+^{d}$: 
\begin{align*}
    \lim\limits_{L \to \infty} \frac{K^{(L)}(x, \tilde{x})}{{(A')}^{2L} (L+1)} = \frac{R(\|x - \tilde{x}\|)}{\|x - \tilde{x}\|^{\alpha}}, 
\end{align*}
where $\alpha = -2\frac{\log(A'^2)}{\log\left(B'\right)}$ and $R(\cdot)$ is non-negative, bounded from above, and bounded away from $0$ around $0$. 
\end{theorem}

The full proof is presented in Supplementary Information \ref{appendix: B}, and we outline its key steps in Section \ref{sec: Proof Sketch}.  Theorem \ref{theorem: Optimality of the NTK} below characterizes the activation functions for which the infinitely wide and deep network achieves optimality. In particular, we establish the optimality of the classifier, $m_n(\cdot)$, given by taking the limit as $L \to \infty$ of the kernel machine in Eq.~\eqref{eq: Kernel Machine with NTK} with $K = K^{(L)}$, i.e.
\begin{align}
    \label{eq: NTK Classifier}
    m_n(x)  = \lim_{L \to \infty} \sign\Big( y {\big(K_n^{{(L)}}\big)^{-1}} K^{(L)}(X, x)\Big).
\end{align}

\begin{theorem}
\label{theorem: Optimality of the NTK}
Let $m_n$ denote the classifier in Eq.~\eqref{eq: NTK Classifier} corresponding to training an infinitely wide and deep network with activation function $\phi(\cdot)$ on $n$ training examples.  For $z \sim \mathcal{N}(0, 1)$, define $A = \mathbb{E}[\phi(z)]$, $A' = \mathbb{E}[\phi'(z)]$, and $B' = \mathbb{E}[\phi'(z)^2]$. If
$$A = 0 \quad\textrm{and} \quad  A' \neq 0 \quad\textrm{and}\quad -\frac{\log({A'}^2)}{\log\left(B'\right)} = \frac{d}{2},$$
then this classifier is \emph{Bayes optimal}.\footnote{Formally, this classifier satisfies that for almost all $x \in \mathcal{S}_+^{d}$ and for any $\epsilon > 0$, 
\begin{align*}
   \lim_{n \to \infty} \mathbb{P}_X\left( \left| m_n(x) -  \argmax\limits_{\tilde{y} \in \{-1, 1\}} \mathbb{P}\left( y = \tilde{y} | x \right) \right| > \epsilon \right) = 0 ~.
\end{align*}}
\end{theorem}


While the full proof of Theorem~\ref{theorem: Optimality of the NTK} is presented in Supplementary Information \ref{appendix: B} and \ref{appendix: C}, we outline its key steps in Section~\ref{sec: Proof Sketch}.  In particular, the proof follows by using Theorem~\ref{theorem: NTK singular kernel} above, proving that $m_n$ is a singular kernel classifier, and then using the results of~\cite{DevroyeHilbertKernel}, which establish conditions under which singular kernel estimators achieve optimality.  
The following corollary (proof in Supplementary Information \ref{appendix: D}) presents a concrete class of activation functions that satisfy the conditions of Theorem \ref{theorem: Optimality of the NTK} for any given data dimension $d$.

\begin{corollary}
\label{corollary: example optimal NTK}
Let $m_n$ denote the classifier in Eq.~\eqref{eq: NTK Classifier} corresponding to training an infinitely wide and deep network with activation function  
\begin{align*}
    \phi(x) &= \begin{cases}
    \frac{1}{12 \sqrt{70}} h_7(x) + \frac{1}{\sqrt{2}} x   & \text{if $d = 1$,} \\
    \frac{1}{2^{d/4}} \left(\frac{x^3 - 3x}{\sqrt{6}}\right) + \sqrt{1 - \frac{2}{2^{d/2}}} \left(\frac{x^2 - 1}{\sqrt{2}}\right) +  \frac{1}{2^{d/4}}x & \text{if $d \geq 2$,}  \end{cases}
\end{align*} 
where $h_7(x)$ is the $7$\textsuperscript{th} probabilist's Hermite polynomial.\footnote{For $d = 1$, this activation function can be written in closed form as $\frac{x^7 - 21x^5 + 105x^3 + (12 \sqrt{35} -105)x}{12\sqrt{70}}$.} Then the classifier $m_n$ is Bayes optimal.
\end{corollary}


We note the remarkable simplicity of the above activation functions yielding infinitely wide and deep networks that achieve optimality.  In particular, for $d \geq 2$, these activations are simply cubic polynomials. 




\subsection*{Case 2 ($A = 0, A' = 0$) networks implement 1-NN.}
\label{sec: 1 NN NTK}


We now identify conditions on the activation function under which infinitely wide and deep networks implement the 1-NN classifier. 

\begin{theorem}
\label{theorem: NTK 1NN}
Let $m_n$ denote the classifier in Eq.~\eqref{eq: NTK Classifier} corresponding to training an infinitely wide and deep network with activation function $\phi(\cdot)$ on $n$ training examples.  For $z \sim \mathcal{N}(0, 1)$, define $A = \mathbb{E}[\phi(z)]$ and $A' = \mathbb{E}[\phi'(z)]$. If $A = A' = 0$, then $m_n(x)$ implements 1-NN classification on $\mathcal{S}_+^{d}$.
\end{theorem}
The proof of Theorem \ref{theorem: NTK 1NN} is provided in Supplementary Information \ref{appendix: E}.  The proof strategy is to show that the value of the kernel between a test example and its nearest training example dominates the prediction as $L\to\infty$.  In particular, assuming without loss of generality that $x^Tx^{(1)} > x^Tx^{(j)}$ for $j \in \{2, 3, \ldots, n\}$, we prove that: 
\begin{align*}
    \lim_{L \to \infty} \frac{K^{(L)}(x, x^{(j)})}{K^{(L)}(x, x^{(1)})} = 0~.
\end{align*}
As a result, after re-scaling by $K^{(L)}(x, x^{(1)})$, we obtain that $m_n(x) = \sign(y^{(1)})$.  We note that this proof is analogous to the standard proof that the Gaussian kernel $K(x, \tilde{x}) = \exp\left(- \gamma \|x - \tilde{x}\|^2 \right)$ converges to the 1-NN classifier as  $\gamma \to \infty$.

\subsection*{Case 3 ($A \neq 0$) networks implement majority vote classifiers.}


We now analyze infinitely wide and deep networks 
when the activation functon satisfies $\mathbb{E}[\phi(z)] \neq 0$ for $z \sim \mathcal{N}(0, 1)$.  In this setting, we establish conditions under which the infinitely wide and deep network implements majority vote classification, i.e., the prediction on test samples is simply the label of the class with greatest representation in the training set. More precisely, the following proposition (proof in Supplementary Information \ref{appendix: F}) implies that when the infinite depth NTK is a constant non-zero value for any two non-equal inputs, the resulting classifier is the majority vote classifier.  

\begin{prop}
\label{prop: NTK Majority Vote Condition}
Let $m_n$ denote the classifier in Eq.~\eqref{eq: NTK Classifier} corresponding to training an infinitely wide and deep network with activation function $\phi(\cdot)$ on $n$ training examples. For any $x, \tilde{x} \in \mathcal{S}_+^{d}$ with $x \neq \tilde{x}$, if the NTK $K^{(L)}$ satisfies 
\begin{align}
    \label{eq: NTK Majority vote condition 1}
    \lim_{L \to \infty} \frac{K^{(L)}(x, \tilde{x})}{C(L)} = C_1 \quad \textrm{and } \quad \lim_{L \to \infty} \frac{K^{(L)}(x, \tilde{x})}{C(L)} \neq \lim_{L \to \infty} \frac{K^{(L)}(x, x)}{C(L)},
\end{align}
with $C_1 > 0$ and $0 < C(L) < \infty$ for any $L$, then $m_n$ implements the majority vote classifier, i.e., 
\begin{align*}
    m_n(x) = \sign \Big(\sum_{i=1}^{n} y^{(i)} \Big)~.
\end{align*}
\end{prop}


We now analyze which activation functions satisfy Eq.~\eqref{eq: NTK Majority vote condition 1}.  As described in~\cite{PooleTransientChaos, YangEdgeofChaos, JudithEdgeofChaos, JaschaOrderedChaoticPhase}, under case 3, the value of $B' = \mathbb{E}[\phi'(z)^2]$ for $z \sim \mathcal{N}(0, 1)$ determines the fixed point attractors of $K^{(L)}$ as $L \to \infty$.  Thus, the infinite depth behavior under case 3 can be broken down into three cases based on the value of $B'$.  Using the terminology from \cite{PooleTransientChaos}, these cases are: 
$$(i)\; B' > 1 \textrm{ (Chaotic Phase),} \quad (ii)\; B' < 1 \textrm{ (Ordered Phase),} \quad  (iii)\; B' = 1  \textrm{ (Edge of Chaos)}.$$ 
In Lemma \ref{lemma: chaotic phase majority vote} in Supplementary Information \ref{appendix: G}, we demonstrate that in the chaotic phase, the resulting infinite depth NTK satisfies the conditions of Proposition \ref{prop: NTK Majority Vote Condition} and thus implements the majority vote classifier.  In Lemma \ref{lemma: ordered phase majority vote} in Supplementary Information \ref{appendix: G}, we similarly show that in the ordered phase the infinite depth NTK also corresponds to the majority vote classifier.\footnote{More precisely, we consider the behavior of the infinite depth classifier under ridge-regularization, as the regularization term approaches $0$.}  The remaining case known as "edge of chaos" has been analyzed in prior works for specific activation functions;  for example, 
the NTK for networks with ReLU activation 
satisfies Eq.~\eqref{eq: NTK Majority vote condition 1} with $C_1 = \frac{1}{4}$ and $C(L) = L+1$~\cite{ResNetNTK, JudithNTK}.  Hence by Proposition \ref{prop: NTK Majority Vote Condition}, the corresponding infinite depth classifier for ReLU networks corresponds to the majority vote classifier.



\section{Outline of Proof Strategy for Theorems \ref{theorem: NTK singular kernel} and \ref{theorem: Optimality of the NTK}}
\label{sec: Proof Sketch}

In the following, we outline the proof strategy for our main results.  This involves analyzing infinitely wide and deep networks via the limiting NTK kernel given by $K^{(L)}$ as the number of hidden layers $L \to \infty$.  As shown in~\cite{NTKJacot}, $K^{(L)}$ can be written recursively in terms of $K^{(L-1)}$ and the so-called dual activation function, which was introduced in~\cite{DualActivation}.  
\vspace{3mm}

\begin{definition} 
Let $\phi: \mathbb{R} \to \mathbb{R}$ be an activation function satisfying $\mathbb{E}_{x \sim \mathcal{N}(0, 1)}[\phi(x)^2] < \infty$. Its \textbf{dual activation function} $\check{\phi}: [-1, 1] \to \mathbb{R}$ is given by
\begin{align*}
    \check{\phi}(z) &= \mathbb{E}_{(u, v) \sim \mathcal{N}\left(\mathbf{0}, \Lambda \right)} [\phi(u) \phi(v)], \quad \textrm{where }
    \Lambda = \begin{bmatrix} 1 & z \\ z & 1\end{bmatrix}.
\end{align*}
\end{definition}

While all quantities in our theorems are stated in terms of activation functions, these can be restated in terms of dual activations as follows: 
\begin{align*}
    A^2  &= \check{\phi}(0) \quad \text{and} \quad (A')^2 = \check{\phi'}(0)  ~ \quad \text{and} \quad B' = \check{\phi'}(1) ~.
\end{align*}

Assuming that $\phi$ is normalized such that $\check{\phi}(1) = 1$,\footnote{Such normalization is always possible for any activation function satisfying $\mathbb{E}[\phi(z)^2] < \infty$ for $z \sim \mathcal{N}(0, 1)$ and has been used in various works before including~\cite{JudithNTK, ResNetNTK, LaplaceAndNTK1, MehlersFormulaCompositionalKernels, PenningtonEdgeofChaos, JudithEdgeofChaos}.}  the recursive formula for the NTK of a deep fully connected network for data on the unit sphere was described in~\cite{NTKJacot, LaplaceAndNTK1} in terms of dual activation functions as follows.
\vspace{3mm}

\noindent \textbf{Recursive Formula for the NTK.}
\label{lemma: Kernels radial recursion}
Let $f^{(L)}(x ; \mathbf{W})$ denote a fully connected neural network with $L$ hidden layers and activation $\phi(\cdot)$.  For $x, \tilde{x} \in \mathcal{S}^{d}$, let $z = x^T \tilde{x}$.  Then $K^{(L)}$ is radial, i.e. $K^{(L)}(x, \tilde{x}) = K^{(L)}(z)$, with
\begin{align}
\label{eq: NTK Recursive Formula}
    K^{(L)}(z) &= \check{\phi}^{(L)}(z) + K^{(L-1)}(z) \check{\phi}' \big(\check{\phi}^{(L-1)}(z)\big) \quad \textrm{and} \quad K^{(0)}(z) = z, 
\end{align}
where $\check{\phi}^{(L)}(z) = \check{\phi}( \check{\phi}^{(L-1)}(z))$ with $\check{\phi}^{(0)}(z) = \check{\phi}(z)$ and $\check{\phi}'(\cdot)$ denotes the derivative of $\check{\phi}(\cdot)$.  
\vspace{0.4cm}

We utilize the dynamical system in Eq.~\eqref{eq: NTK Recursive Formula} to analyze the behavior of $K^{(L)}(\cdot)$ as $L \to \infty$. Theorem \ref{theorem: NTK singular kernel} implies that upon normalization by $(L+1)\check{\phi'}(0)^L$, this dynamical system converges to a singular kernel with singularity of order $\alpha = -\log \left(\check{\phi'}(0) \right)/\log \left(\check{\phi'}(1) \right)$.  We now present a sketch of the proof of this result.  

We first derive the order of the singularity upon iteration of $\check{\phi}$, since as we show in Supplementary Information \ref{appendix: B}, the order of the singularity of the infinite depth NTK is the same as that of the iterated $\check{\phi}$.  Since we consider data in $\mathcal{S}_+^{d}$, $\check{\phi}(\cdot)$ is a function defined on the unit interval $[0, 1]$.  Hence, understanding the properties of infinitely wide and deep networks reduces to understanding the properties of iterating a function on the unit interval.  To provide intuition around how the iteration of a function on the unit interval can give rise to a function with a singularity, we discuss iterating a piecewise linear function as an illuminating example; see Fig.~\ref{fig: Piecewise Linear Iteration} for a visualization. 

\begin{lemma} 
\label{lemma: piecewise linear iteration}
For $0 < a < 1$ and $b > 1$, let $f: [0, 1] \to \mathbb{R}$ and $c = \frac{b-1}{b-a}$ such that
\begin{align*}
    f(x) = \begin{cases} ax & \textrm{if } x \in [0, c] \\ 1 - b(1-x)  & \textrm{if } x \in (c, 1]\end{cases}~. 
\end{align*}
Then, 
\begin{align*}
    \lim_{L \to \infty} \frac{f^{(L)}(x)}{a^{L}} = \frac{R(x)}{(1 - x)^{-\log_b a}}~, 
\end{align*}
where $R(x)$ is non-negative, bounded from above and bounded away from $0$ around $x = 1$. 
\end{lemma}
\begin{proof}
For any $x \in [0, c]$, we necessarily have: 
\begin{align*}
    \lim_{L \to \infty} \frac{f^{(L)}(x)}{a^L} = \lim_{L \to \infty} \frac{a^L x}{a^L} = x.
\end{align*}
Now for fixed $x \in (c, 1)$, since $x = 0$ is an attractive fixed-point of $f$, let $L_0$ denote the smallest integer such that $f^{(L_0)}(x) \leq c$.  Hence, since $f^{(L_0)}(x) \in [0, c]$, we obtain: 
\begin{align}
\label{eq: Iteration of linear}
     \lim_{L \to \infty} \frac{f^{(L)}(x)}{a^L} = \lim_{L \to \infty} \frac{f^{(L - L_0)}(f^{(L_0)}(x))}{a^{L - L_0}} \frac{1}{a^{L_0}} = f^{(L_0)}(x) a^{-L_0}.
\end{align}
We next solve for $L_0$ by analyzing the iteration of $g(x) := 1 - b(1-x)$.  In particular, we observe that $g^{(L)}(x) = 1 - b^L(1-x)$, and thus $L_0$ is given by: 
\begin{align*}
    1 - b^{L_0}(1 - x) \leq c \implies L_0 = \left\lceil \log_a \left(\frac{1-x}{1-c}\right)^{-\log_b a} \right\rceil \implies a^{-L_0} \in \left[ \left(\frac{1-c}{1-x}\right)^{-\log_b a}, \frac{1}{a} \left(\frac{1-c}{1-x}\right)^{-\log_b a} \right]. 
\end{align*}
Hence, by Eq.~\eqref{eq: Iteration of linear} we conclude that for $x \in (c, 1)$, it holds that 
\begin{align*}
    \lim_{L \to \infty} \frac{f^{(L)}(x)}{a^{L}} &= \frac{R(x)}{(1 - x)^{-\log_b a}} ~,
\end{align*}
where $R(x)$ is non-negative, bounded from above and bounded away from $0$ around $x=1$, which completes the proof.
\end{proof}


\begin{figure}[!t]
    \centering
    \includegraphics[width=.9\textwidth]{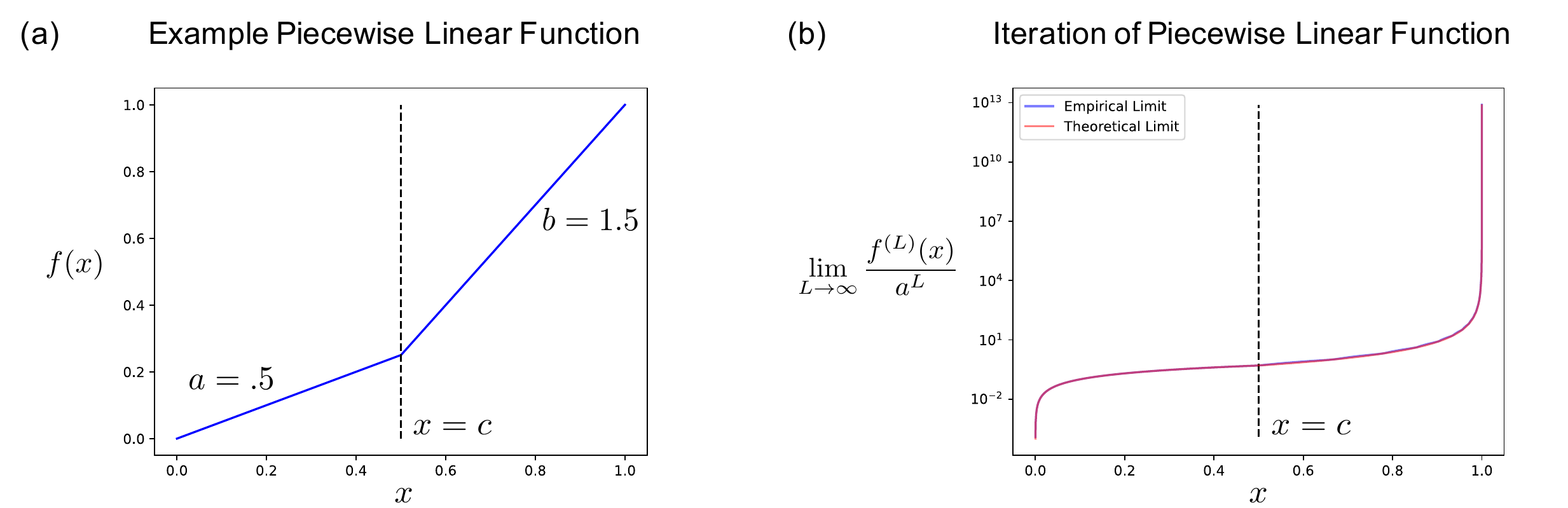}
    \caption{Iteration of a piecewise linear function on a unit interval leads to a function with a singularity at $x=1$, upon appropriate normalization.  (a) We consider the piecewise linear function $f(x)$ given by $1 - b(1-x)$ on $(c, 1]$ and $ax$ on $[0, c]$, where $a =.5, b = 1.5$ and $c = \frac{b-1}{b-a}$. (b) We observe that upon iterating $f(\cdot)$ numerically to the limit of machine precision, the resulting function strongly agrees with the theoretical limit of Lemma~\ref{lemma: piecewise linear iteration} given by a function with singularity of order $-\log_b a \approx 1.7$.}
    \label{fig: Piecewise Linear Iteration}
\end{figure}

In Supplementary Information \ref{appendix: B}, we extend this analysis to the iteration of dual activations on the unit interval, thereby establishing the order of a singularity obtained by iterating dual activation functions.  We then show that this order equals the order of the singularity given by the infinite depth NTK.   

Next, we discuss the proof strategy for Theorem \ref{theorem: Optimality of the NTK}, which establishes conditions on the activation function under which infinitely wide and deep networks achieve optimality in the classification setting.  The proof builds on results in~\cite{DevroyeHilbertKernel} 
characterizing the optimality of singular kernel smoothers of the form 
$$g(x) = \frac{\sum_{i=1}^{n} y^{(i)} K(x^{(i)}, ~x)}{\sum_{i=1}^{n} K(x^{(i)},~x)}, \quad \textrm{where } K(x^{(i)}, x) = \frac{1}{\|x - x^{(i)}\|^{\alpha}}.$$  
In particular, it is shown that if $\alpha = d$, then $g(x)$ achieves optimality.  Since Theorem \ref{theorem: NTK singular kernel} establishes conditions under which the infinite depth NTK implements a singular kernel, to complete the proof we show that infinitely wide and deep classifiers achieve optimality by (1) showing that the classifier $m_n$ implements a singular kernel smoother, and (2) selecting $\phi$ such that $\alpha = d$ for the corresponding singular kernel.

\section{Discussion}

In this work, we identified and constructed explicit neural networks that achieve optimality for classification when trained using standard procedures.  Furthermore, we provided a taxonomy characterizing the behavior of infinitely wide and deep neural network classifiers.  Namely, we showed that these models implement one of the following three well-known types of classifiers: (1) 1-NN (test predictions are given by the label of the nearest training example) ; (2) majority vote (test predictions are given by the label of the class with greatest representation in the training set); or (3) singular kernel classifiers (a set of classifiers containing those that achieve optimality).  We conclude by discussing implications of our work and future extensions.

\textbf{Benefit of Depth in Neural Networks.}  An emerging trend in machine learning is that larger neural networks capable of interpolating (i.e., perfectly fitting) the training data, can generalize to test data~\cite{belkin2019reconciling, DeepDoubleDescent, RethinkingGeneralization}.  While the size of neural networks can be increased through width or depth, works such as~\cite{belkin2019reconciling, DeepDoubleDescent} primarily identified a benefit to increasing network width.  Indeed, it remained unclear whether there was any benefit to using extremely deep networks.  For example, recent works~\cite{EshaanConvDepth, XiaoDynamicalIsometry, PenningtonEdgeofChaos} empirically demonstrated that drastically increasing depth in networks with ReLU or tanh activation could lead to worse performance.  In this work, we established a remarkable benefit of very deep networks by proving that they achieve optimality with a careful choice of activation function.  In line with previous empirical findings, we proved that deep networks with activations such as ReLU or tanh do not achieve optimality.



\textbf{Regression versus Classification.}  Our results demonstrate the benefit of using infinitely wide and deep networks for classification tasks.  We note that this in stark contrast to the regression setting, where infinitely deep and wide neural networks are far from optimal, as they simply predict a non-negative constant almost everywhere~\cite{JudithNTK, ResNetNTK}.  Thus, our work provides concrete examples of neural networks that are effective for classification but not regression.  



\textbf{Edge of Chaos Regime.}  An interesting class of models that are only partially characterized by our taxonomy corresponds to networks with activations in the edge of chaos regime, i.e., when the activation function, $\phi(\cdot)$ satisfies $\mathbb{E}[\phi(z)] \neq 0$ and $\mathbb{E}[\phi'(z)^2] = 1$ for $z \sim \mathcal{N}(0, 1)$.  We proved that all activations in this class that have been described so far~\cite{JudithNTK, ResNetNTK}, including the popular ReLU activation, give rise to infinitely wide and deep networks that implement the majority vote classifier.  While it appears that all activations in this class lead to the majority vote classifier, it remains open to understand whether there exist other activations in this regime that implement alternative classifiers.



\textbf{Finite vs. Infinite Neural Networks.} In this work, we identified and constructed infinitely wide and deep classifiers that achieve optimality.  An important next question is to understand whether interpolating neural networks that are finitely wide and deep can achieve optimality for classification and provide specific activation functions to do so.  
We also note that Bayes optimality considers the setting when the number of training examples approaches infinity. Another natural next step is to characterize the number of training examples needed for infinitely wide and deep classifiers to reasonably approximate the Bayes optimal classifier.  
Recent work~\cite{RatesforHilbertKernel} identified a slow (logarithmic) rate of convergence for singular kernel classifiers, thereby implying that many training examples are needed for these models to be effective in practice.  An important open direction of future work is thus to determine not only whether finitely wide and deep networks are optimal for classification but also whether these models require fewer samples to perform well in practice.

\section*{Acknowledgements}

A.R. and C.U.~were partially supported by NSF (DMS-1651995), ONR (N00014-17-1-2147 and N00014-18-1-2765), the MIT-IBM Watson AI Lab, \textcolor{black}{the Eric and Wendy Schmidt Center at the Broad Institute,} and a Simons Investigator Award (to C.U.). M.B.~acknowledges support from NSF  IIS-1815697 and  NSF DMS-2031883/Simons Foundation  Award 814639.

\bibliographystyle{abbrv}
\bibliography{references}

\newpage 

\appendix

\noindent{\huge\bf{Supplementary Information}}

\section{Preliminaries on NTK and Dual Activations}
\label{appendix: A}

In this section, we briefly review properties of dual activations that we will use to prove our main results.  In order to analyze the behavior of the iterated dual activation, we reference the following result of~\cite{DualActivation}, which implies that the dual activation is analytic around $0$ on the interval $[-1, 1]$.  

\vspace{5mm}
\noindent \textbf{Analyticity of Dual Activations.}
Let $\phi(\cdot)$ be an activation function such that $\mathbb{E}_{x \sim \mathcal{N}(0, 1)}[\phi(x)^2] = 1$, and let $\check{\phi}:[-1, 1] \to \mathbb{R}$ denote the dual activation. Then, for $z \in [-1,1]$,
\begin{align}
\label{eq: Analytic formula dual activation}
    \check{\phi}(z) = \sum_{i=0}^{\infty} a_i z^i, 
\end{align}
where $a_i \geq 0$ for all $i \in \mathbb{N}$. 

\vspace{3mm}

As proven in \cite[Lemma 11]{DualActivation}, several key properties are implied by Eq.~\eqref{eq: Analytic formula dual activation}. Those utilized in this work are: (1) $\check{\phi}$ is increasing on $[0, 1]$, and (2) non-negativity of $\check{\phi}(\cdot)$ on $[0, 1]$.   Eq.~\eqref{eq: Analytic formula dual activation} also implies the following property of dual activations that we will use to construct our taxonomy of infinitely wide and deep neural network classifiers.  

\begin{lemma}
\label{lemma: 1st derivative of dual values}
Let $\check{\phi}: [-1, 1] \to \mathbb{R}$ be a dual activation such that $\check{\phi}(0) = 0$, $\check{\phi}(1) = 1$, and $\check{\phi}(z) \neq z$.  Then, $0 \leq \check{\phi}'(0) < 1$.
\end{lemma}

\begin{proof}
By Eq.~\eqref{eq: Analytic formula dual activation}, we need only show that $0 \leq a_1 < 1$.  Since $\check{\phi}(1) = 1$, we obtain that $\sum_{i=1}^{\infty} a_i = 1$.   Since $a_i \geq 0 $ for all $i \in \mathbb{N}$, we conclude that $0 \leq a_1 \leq 1$.  Now if $a_1 = 1$, then $a_i =  0$ for $i \geq 2$, which implies that $\check{\phi}(z) = z$.  Hence, we conclude that $0 \leq a_1 < 1$, which completes the proof.  
\end{proof}

\section{Proofs of Theorem \ref{theorem: NTK singular kernel} and Theorem \ref{theorem: Optimality of the NTK}}
\label{appendix: B}

\noindent We first prove Theorem \ref{theorem: NTK singular kernel}, which is expressed below in terms of the dual activation function. 

\begin{theorem*}
Let $K^{(L)}$ denote the NTK of a fully connected neural network with $L$ hidden layers and activation function $\phi(\cdot)$.  For $x, \tilde{x} \in \mathcal{S}_+^{d}$, let $z = x^T \tilde{x}$.  If the dual activation function $\check{\phi}(\cdot)$ satisfies
\begin{enumerate}
\item[1)] $\check{\phi}(0) = 0$, $\check{\phi}(1) = 1$,
\item[2)] $0 < \check{\phi}'(0) < 1$ and $\check{\phi}'(1) < \infty$,
\end{enumerate}
then: 
\begin{align*}
    \lim\limits_{L \to \infty} \frac{K^{(L)}(x, \tilde{x})}{\check{\phi}'(0)^L (L+1)} = \frac{R(x^T \tilde{x})}{\|x - \tilde{x}\|^{\alpha}}~, 
\end{align*}
where $\alpha = -2\frac{\log(\check{\phi}'(0))}{\log\left(\check{\phi}'(1)\right)}$ and $R(u) \geq 0$ is bounded for $u \in [0, 1]$ and bounded away from $0$ around $u = 1$. 
\end{theorem*}

In order to prove this theorem, we first prove that the iterated, normalized NTK converges to a singular kernel without explicitly identifying the order of the singularity.       

\begin{lemma}
\label{lemma: convergence to singular kernel}
Let $K^{(L)}$ denote the NTK of a depth $L$ fully connected network with normalized activation function $\phi$.  Assuming $\check{\phi}$ satisfies the conditions of Theorem \ref{theorem: NTK singular kernel}, then for any $x, \tilde{x} \in \mathcal{S}_+^d$ it holds that
\begin{align*}
    \lim\limits_{L \to \infty} \frac{K^{(L)}(x, \tilde{x})}{a_1^{L} (L+1)} = \psi(x^T\tilde{x}),
\end{align*}
where $\psi: [0, 1] \to \mathbb{R}$ can be written as a power series with non-negative coefficients with a singularity at $1$.
\end{lemma}
\begin{proof}
We utilize the form of the NTK given in~\cite{CNTKArora} and utilize the radial form of the kernel in Eq.~\eqref{eq: NTK Recursive Formula}.  Namely, for $z \in [0, 1]$, we have: 
\begin{align}
\label{eq: Closed form for NTK}
    K^{(L)}(z)= \sum_{i=0}^{L} \check{\phi}^{(i)}(z) \prod_{j=i}^{L-1} \check{\phi}'\left(\check{\phi}^{(j)}(z)\right),
\end{align}
where $\check{\phi}^{(i)}$ denotes the iteration of $\check{\phi}$ $i$ times.  By Eq.~\eqref{eq: Analytic formula dual activation} and since $\check{\phi}(0) = 0$, we have that $\check{\phi}(z) = \sum_{i=1}^{\infty} a_i z^i$ for all $z \in [0, 1]$.\footnote{Note that the sum starts from $a_1$ since $\check{\phi}(0) = 0 \implies a_0 = 0$.}    Now, we bound $\check{\phi}$ by quadratic functions in $z$ and bound $\check{\phi}'$ by linear functions in $z$.  In particular, using the conditions $\check{\phi}(1) = 1$ and $\check{\phi}'(1) = C < \infty$, we obtain the upper bounds: 
\begin{align*}
    \check{\phi}(z) &= a_1 \left(z + \sum_{i=2}^{\infty} \frac{a_i}{a_1} z^{i}\right) \leq a_1 \left(z + \sum_{i=2}^{\infty} \frac{a_i}{a_1} z^{2}\right) = a_1 \left(z +  \left(\frac{1}{a_1} - 1\right) z^{2} \right), \\
    \check{\phi}'(z) &= a_1 \left(1 + \sum_{i=2}^{\infty} i \frac{a_i}{a_1} z^{i-1}\right) \leq a_1 \left(1 + \sum_{i=2}^{\infty} \frac{a_i}{a_1} i z \right) = a_1 \left(1 +  \left(\frac{C}{a_1} - 1\right) z \right).     
\end{align*}
Similarly, we obtain the lower bounds: 
\begin{align*}
    \check{\phi}(z) &= a_1 \left(z + \sum_{i=2}^{\infty} \frac{a_i}{a_1} z^{i}\right) \geq a_1 \left(z +  \frac{a_2}{a_1} z^{2}\right),  \\
    \check{\phi}'(z) &= a_1 \left(1 + \sum_{i=2}^{\infty} i \frac{a_i}{a_1} z^{i-1}\right) \geq a_1 \left(1 +  \frac{2a_2}{a_1} z \right) \geq a_1 \left(1 +  \frac{a_2}{a_1} z\right). 
\end{align*}
Now, substituting the above lower and upper bounds into the recursion for $\check{\phi}^{(i)}$, we obtain 
\begin{align}
\label{eq: Quadratic recursion for iterated dual}
a_1^{i} z \prod_{j=0}^{i-1} \left( 1 + \frac{a_2}{a_1}\check{\phi}^{(j)}(z) \right) \leq \check{\phi}^{(i)}(z) \leq a_1^{i} z \prod_{j=0}^{i-1} \left( 1 + \left(\frac{1}{a_1} - 1\right)\check{\phi}^{(j)}(z) \right).
\end{align}
Lastly, since $C \geq 1$, substituting Eq.\eqref{eq: Quadratic recursion for iterated dual} and the bounds on $\check{\phi}'$ into Eq.~\eqref{eq: Closed form for NTK} for $K^{(L)}$, we obtain 
\begin{align*}
    (L+1) a_1^{L} z \prod_{j=0}^{L-1}\left(1 + \frac{a_2}{a_1}\check{\phi}^{(j)}(z)\right) \leq K^{(L)}(z) \leq (L+1) a_1^{L} z \prod_{j=0}^{L-1}\left(1 + \left(\frac{C}{a_1} - 1\right)\check{\phi}^{(j)}(z)\right).
\end{align*}
Hence, to prove that $\psi(z):= \lim\limits_{L \to \infty} \frac{K^{(L)}(z)}{a_1^{L} (L+1)}$ is finite for $z \in [0, 1)$, we need to show that
\begin{align*}
    \prod_{j=0}^{\infty} \left(1 + \tilde{C}\check{\phi}^{(j)}(z)\right) < \infty 
\end{align*}
for all $z \in [0, 1)$ and any constant $\tilde{C}$.  By the Cauchy criterion \cite[Ch.5]{SteinComplexAnalysis}, the above infinite product converges if and only if the following sum converges:
\begin{align*}
    \sum_{j=0}^{\infty} \tilde{C} \check{\phi}^{(j)}(z) < \infty.
\end{align*}
This sum converges by the ratio test.  In particular, 
\begin{align*}
    \lim\limits_{j \to \infty} \frac{\check{\phi}^{(j)}(z)}{\check{\phi}^{(j-1)}(z)} = \lim\limits_{z \to 0} \frac{\check{\phi}(z)}{z}  = a_1 < 1,
\end{align*}
where we used the contractive mapping theorem~\cite{DynamicalSystem} to establish the first equality, since $0$ is a fixed point attractor of $\check{\phi}$.  As a consequence, $\psi(z) < \infty$ for $z \in [0, 1)$.  Now according to Eq.~\eqref{eq: Closed form for NTK}, $\psi(z)$ can be written as a  convergent power series with non-negative coefficients  for $z \in [0, 1)$.  To establish the singularity of $\psi(z)$ at $z = 1$, we show that for any constant $R > 0$, there exists $z_0$ such that $\psi(z) > R$ for $z > z_0$.  In particular, note that for any fixed $L_0$, 
\begin{align*}
    \lim\limits_{L \to \infty} \frac{K^{(L)}(z)}{a_1^{L} (L+1)} = \psi(z) \geq z \prod_{j=0}^{L_0-1}\left(1 + \frac{a_2}{a_1}\check{\phi}^{(j)}(z)\right).
\end{align*}
The right-hand side is a continuous function with maximum value $\left(1 + \frac{a_2}{a_1}\right)^{L}$.  Hence, by selecting $L_0$ such that  $\left(1 + \frac{a_2}{a_1}\right)^{L_0} > R$, we can then pick $z_0$ such that $\psi(z) > R$ for all $z > z_0$. Hence, we conclude that 
\begin{align*}
    \lim\limits_{L \to \infty} \frac{K^{(L)}(z)}{a_1^{L} (L+1)} = \psi(z),
\end{align*}
where $\psi(z)$ can be written as a convergent power series with non-negative coefficients on $[0, 1)$ with a singularity at $z = 1$, which completes the proof.
\end{proof}

We will now prove Theorem \ref{theorem: NTK singular kernel} by establishing the order of the singularity of $\psi$ from Lemma \ref{lemma: convergence to singular kernel}.  To characterize the order of this singularity, we will generally characterize the order of the singularity arising from iterating functions on the interval $[0, 1]$.  In particular, we begin by establishing the order of the singularity of the normalized iteration of a function that is linear around $x = 1$.  
\begin{lemma}
\label{lemma: NNGP linear around 1}
Let 
\begin{align*}
f(z) = \begin{cases} g(z) ~~ \text{if $z \in [0, d]$} \\ 1 - b(1 - z) ~~ \text{if $z \in (d, 1]$}  \end{cases},
\end{align*}
with $d < 1$ such that $f(z)$ is strictly monotonically increasing and $g(z)$ can be written as a convergent power series with non-negative coefficients with $g(0) = 0$, $g'(0) = a < 1$,  and $b > 1$.  Then for $z \in (d, 1]$, it holds that
$$\lim_{L \to \infty} \frac{f^{(L)}(z)}{a^{L}} = \frac{R(z)}{(1 - z)^{-\log_b a}},$$
where $R(z)$ is non-negative for $z \in [0, 1]$, bounded from above, and bounded away from $0$ around $z = 1$.
\end{lemma}
\begin{proof}
We first visualize the curve $f(z)$ in Fig.~\ref{fig: Bounding Curves}a.  For any $z \in (d, 1]$, let $L_0(x)$ denote the smallest number of iterations until $f^{(L)}(z) = z' \leq d$.  Then for $z \in (d, 1)$, we have that
\begin{align*}
    \lim_{L \to \infty} \frac{f^{(L)}(z)}{a^{L}} &= \lim_{L \to \infty} \frac{f^{(L - L_0(z))}(z')}{a^{L - L_0(z)}} a^{- L_0(z)}.
\end{align*}
Now by the proof of Lemma \ref{lemma: convergence to singular kernel}, we know that 
\begin{align*}
    \lim_{L \to \infty} \frac{f^{(L - L_0(z))}(z')}{a^{L - L_0(z)}} = \tilde{R}(z') ~,
\end{align*}
with $\tilde{R}(z') \geq z'$.  Thus, we need only analyze the term  $a^{-L_0(z)}$ to determine the pole order.  In particular, we have that $L_0(z)$ is the least integer that satisfies: 
\begin{align*}
    1 - b^{L_0(z)}(1 - z) \leq d.
\end{align*}
Hence, $L_0(z)$ is given by: 
\begin{align*}
    L_0(z) &= \left\lceil \log_b\left(\frac{1 - d}{1 - z} \right)~ \right\rceil \\
    &= \left\lceil \log_a\left(\frac{1 - d}{1 - z} \right)^{\frac{1}{\log_a b}}~ \right\rceil \\
    &= \left\lceil \log_a\left(\frac{1 - z}{1 - d} \right)^{-\log_b a}~ \right\rceil \\
    &\in \left[ \log_a\left(\frac{1 - z}{1 - d} \right)^{-\log_b a}~ , ~~\log_a\left(\frac{1 - z}{1 - d} \right)^{-\log_b a} + 1 \right].
\end{align*}
As a consequence,
\begin{align*}
    a^{-L_0(z)} \in \left[ \left(\frac{1 - d}{1 - z} \right)^{-\log_b a}, \frac{1}{a} \left(\frac{1 - d}{1 - z} \right)^{-\log_b a} \right].
\end{align*}
Thus we conclude that for $z \in (d, 1)$: 
\begin{align*}
    \lim_{L \to \infty} \frac{f^{(L)}(z)}{a^{L}} &= \frac{R(z)}{(1 - z)^{-\log_b a}} ~,
\end{align*}
where $R(z)$ is non-negative for $z \in [0, 1]$, bounded from above, and bounded away from $0$ around $z = 1$, which concludes the proof.
\end{proof}

\begin{figure*}[!t]
    \centering
    \includegraphics[width=.9\textwidth]{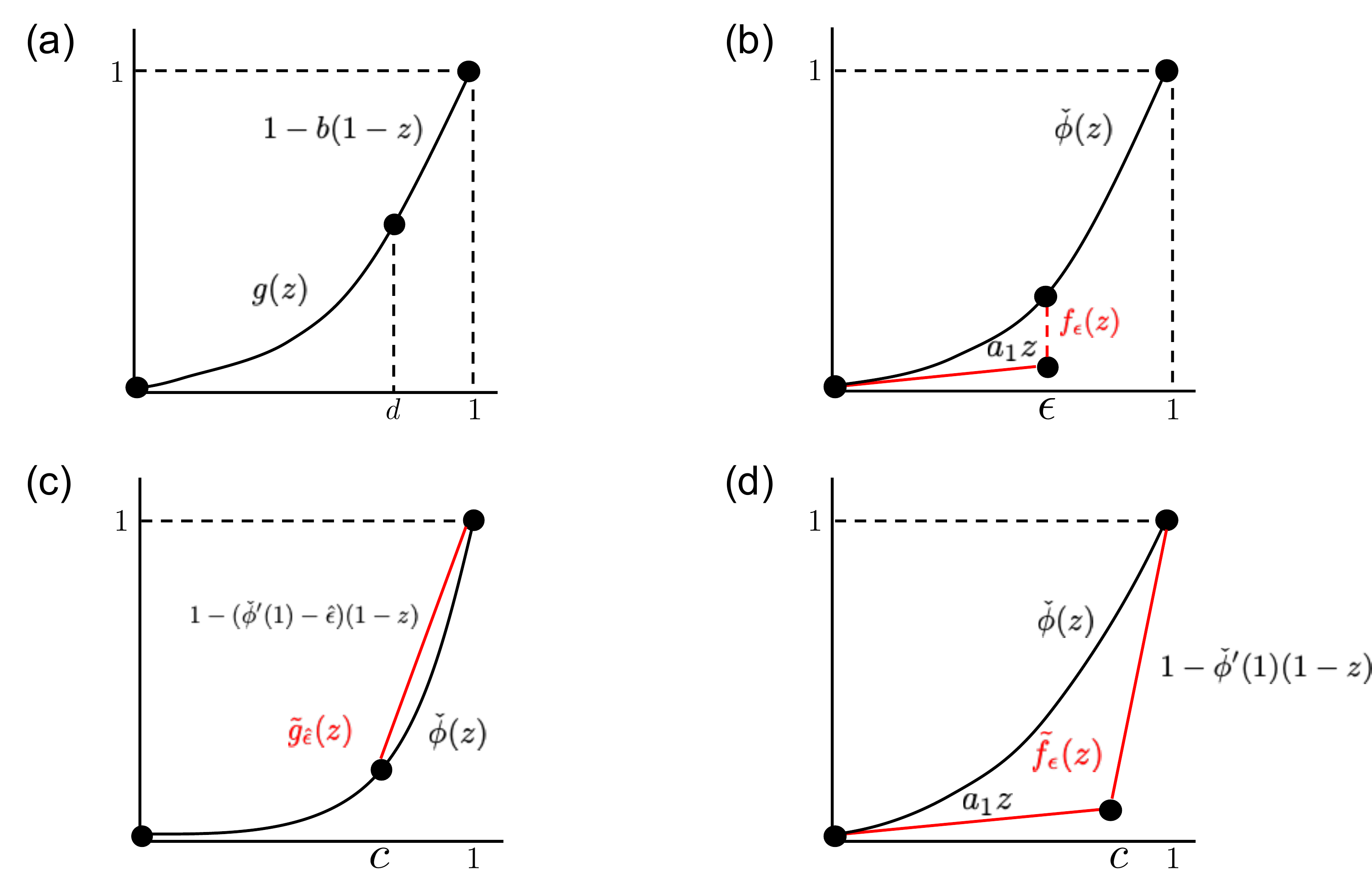}
    \caption{A visualization of the four functions bounding $\check{\phi}(z)$ that are used to prove Theorem \ref{theorem: NTK singular kernel}.}
    \label{fig: Bounding Curves}
\end{figure*}

We will now utilize Lemma \ref{lemma: NNGP linear around 1} to prove Theorem \ref{theorem: NTK singular kernel}. 

\begin{proof}
Let $\check{\phi}(z) = \sum_{i=1}^{\infty}a_i z^{i}$. We will lower bound the dual activation $\check{\phi}$ and its derivative by the piecewise functions: 
\begin{align*}
    f_{\epsilon}(z) = \begin{cases} a_1 z ~~ \text{if $z \in [0, \epsilon)$}  \\ \check{\phi}(z) ~~ \text{if $z \in [\epsilon, 1]$}  \end{cases} ~~\textrm{and} ~~~ h_{\epsilon}(z) = \begin{cases} a_1 ~~ \text{if $z \in [0, \epsilon)$}  \\ \check{\phi}'(z) ~~ \text{if $z \in [\epsilon, 1]$}  \end{cases}.
\end{align*}
The function $f_{\epsilon}(z)$ is visualized in Fig.~\ref{fig: Bounding Curves}b.  Now consider the function $k_{\epsilon}^{(L)}(z)$ defined as follows:
\begin{align*}
    k_{\epsilon}^{(L)}(z) = k_{\epsilon}^{(L-1)}(z) h_{\epsilon}(f_{\epsilon}^{(L-1)}(z)) + f_{\epsilon}^{(L)}(z).
\end{align*}
By definition, we have that $K^{(L)}(z) \geq k_{\epsilon}^{(L)}(z)$ for all $z \in [0, 1]$.  We will now show that for any $\tilde{\epsilon}$, we can select $k_{\epsilon}$ such that
\begin{align}
\label{eq: approximation of NTK by lower bound}
    \lim\limits_{L \to \infty}  \frac{K^{(L)}(z) - k_{\epsilon}^{(L)}(z)}{a_1^{L}(L+1)}  < \tilde{\epsilon}.
\end{align}
To prove Eq.~\eqref{eq: approximation of NTK by lower bound}, we first consider the updates for $L > L_0$ where $L_0$ is the largest integer such that $k^{(L_0)}(z) = K^{(L_0)}(z)$, $h_{\epsilon}(f_{\epsilon}^{(L_0+1)}(z)) = a_1$, and $f_{\epsilon}^{(L_0)}(z) = \check{\phi}^{(L_0)}(z)$.  We will first prove inductively that for $T \in \mathbb{N}$:
\begin{align}
\label{eq: NNGP approximation}
    \check{\phi}^{(L_0 + T)}(z) - f_{\epsilon}^{(L_0 + T)}(z) \leq C_1(z) \left( \sum_{i=0}^{T - 1} a_1^{2L_0 + T - 1 +  i}\right),
\end{align}
where $C_1(z)$ is a term independent of $T$. We begin with the base case of $T = 1$.  Namely, we have for $z \in (\epsilon, 1)$: 
\begin{align*}
     \check{\phi}^{(L_0 + 1)}(z) - f_{\epsilon}^{(L_0 + 1)}(z) &\leq \sum_{i=1}^{\infty} a_i \left(\check{\phi}^{(L_0)}(z) \right)^i - a_1  f_{\epsilon}^{(L_0)}(z) \\
     &= \sum_{i=2}^{\infty} a_i \left(\check{\phi}^{(L_0)}(z) \right)^i ~~~~ \left(\text{since $f_{\epsilon}^{(L_0)}(z) = \check{\phi}^{(L_0)}(z)$}\right) \\ 
     &\leq a_1^{2L_0} \tilde{C}_1(z) \left( 1 - a_1\right)  ~~~~ \left( \text{where $\tilde{C}_1(z)= \left(\lim\limits_{L \to \infty}\frac{\check{\phi}^{(L)}(z)}{a_1^{L}}\right)^2$}\right) \\
     &= C_1(z) a_1^{2L_0}, 
\end{align*}
which concludes the base case.  Now assume the statement is true for $T = T_0$.  Then for $T = T_0 + 1$, we have: 
\begin{align*}
 \check{\phi}^{(L_0 + T_0 + 1)}(z) - f_{\epsilon}^{(L_0 + T_0 + 1)}(z) &= \sum_{i=1}^{\infty} a_i \left(\check{\phi}^{(L_0 + T_0)}(z) \right)^i - a_1  f_{\epsilon}^{(L_0 + T_0)}(z) \\
 &\leq C_1(z) \left( \sum_{i=0}^{T_0 - 1} a_1^{2L_0 + T_0 +  i}\right) + \sum_{i=2}^{\infty}   a_i \left(\check{\phi}^{(L_0 + T_0)}(z) \right)^i \\
 &\leq  C_1(z) \left( \sum_{i=0}^{T_0 - 1} a_1^{2L_0 + T_0 +  i}\right) + C_1(z) a_1^{2L_0 + 2T_0} \\
 &= C_1(z) \left( \sum_{i=0}^{T_0} a_1^{2L_0 + T_0 +  i}\right), 
\end{align*}
which concludes the proof by induction.  We will next prove inductively that for $T \in \mathbb{N}$: 
\begin{align}
\label{eq: kernel difference}
    K^{(L_0 + T)}(z) - k_{\epsilon}^{(L_0 + T)}(z)  \leq C_1(z) \left( \sum_{i =0}^{T-1} (T-i) a_1^{2L_0 + T - 1 + i} \right) + C_2(z) \left(\sum_{i=0}^{T-2} (L_0 + i + 2) a_1^{2L_0 + T + i} \right),
\end{align}
where $C_1(z), C_2(z)$ are terms independent of $T$.  We begin with the base case of $T = 1$.  Namely, we have for $z \in (\epsilon, 1)$: 
\begin{align*}
     K^{(L_0 + 1)}(z) - k_{\epsilon}^{(L_0 + 1)}(z) &\leq [K^{(L_0)}(z) \check{\phi}'(\check{\phi}^{(L_0)}(z)) - k_{\epsilon}^{(L_0)}(z) h_{\epsilon}(f_{\epsilon}^{(L_0)}(z))] +  [\check{\phi}^{(L_0 + 1)}(z) - f_{\epsilon}^{(L_0+1)}(z)] \\
     &=  \check{\phi}^{(L_0 + 1)}(z) - f_{\epsilon}^{(L_0+1)}(z) \\
    &\leq C_1(z) a_1^{2L_0} ~~ \left(\text{by Eq.~\eqref{eq: NNGP approximation}}\right),
\end{align*}
which concludes the base case.  Now, assume that Eq.~\eqref{eq: kernel difference} holds for $T = T_0$.  Then for $T = T_0 + 1$, we have: 
\begin{align*}
    K^{(L_0 + T_0 + 1)}(z) - k_{\epsilon}^{(L_0 + T_0 + 1)}(z) &\leq [K^{(L_0 + T_0)}(z) \check{\phi}'(\check{\phi}^{(L_0 + T_0)}(z)) - k_{\epsilon}^{(L_0 + T_0)}(z) h_{\epsilon}(f_{\epsilon}^{(L_0 + T_0)}(z))] \\
    & \hspace{8mm} +  [\check{\phi}^{(L_0 + T_0 + 1)}(z) - f_{\epsilon}^{(L_0+T_0 + 1)}(z)] \\
    &= \left[K^{(L_0 + T_0)}(z) \sum_{i=1}^{\infty} i a_i (\check{\phi}^{(L_0 + T_0)}(z))^{i-1} - a_1 k_{\epsilon}^{(L_0 + T_0)}(z) \right] \\
    & \hspace{8mm} + \left[\check{\phi}^{(L_0 + T_0 + 1)}(z) - f_{\epsilon}^{(L_0+T_0 + 1)}(z)\right].
\end{align*}
Next we simplify each term in brackets via the inductive hypothesis.  Let  
\begin{align*}
    S_1 = \left[K^{(L_0 + T_0)}(z) \sum_{i=1}^{\infty} i a_i (\check{\phi}^{(L_0 + T_0)}(z))^{i-1} - a_1 k_{\epsilon}^{(L_0 + T_0)}(z) \right].
\end{align*}
Then, given $\check{\phi}'(1) = C < \infty$, for 
$$C_2(z) = \left(C - a_1\right) \lim\limits_{L \to \infty} \frac{K^{(L)}(z)}{a_1^{L} (L+1)} \lim\limits_{L \to \infty} \frac{\check{\phi}^{(L)}(z)}{a_1^{L}} ~, $$
which is finite by Lemma \ref{lemma: convergence to singular kernel}, we have: 
\begin{align*}
    S_1 &\leq C_1(z) \left( \sum_{i =0}^{T_0-1} (T_0 - i) a_1^{2L_0 + T_0 + i} \right) + C_2(z)  \left(\sum_{i=0}^{T_0-2} (L_0 + i + 2) a_1^{2L_0 + T_0 + 1 + i} \right) \\
    & \hspace{8mm} + K^{(L_0 + T_0)}(z)\sum_{i=2}^{\infty} i a_i (\check{\phi}^{(L_0 + T_0)}(z))^{i-1} \\
    &\leq  C_1(z) \left( \sum_{i =0}^{T_0-1} (T_0 - i)  a_1^{2L_0 + T_0  + i} \right) + C_2(z) \left(\sum_{i=0}^{T_0-2}(L_0 + i + 2)  a_1^{2L_0 + T_0 + 1 + i} \right) \\
    & \hspace{8mm} + C_2(z) a_1^{L_0 + T_0} (L_0 + T_0 + 1) a_1^{L_0 + T_0} \\
    &\leq C_1(z) \left( \sum_{i =0}^{T_0-1} (T_0 - i) a_1^{2L_0 + T_0 + i} \right) + C_2(z)   \left(\sum_{i=0}^{T_0-1} (L_0 + i + 2) a_1^{2L_0 + T_0 + 1 + i} \right).
\end{align*}
Next, let: 
\begin{align*}
S_2 = [\check{\phi}^{(L_0 + T_0 + 1)}(z) - f_{\epsilon}^{(L_0+T_0 + 1)}(z)]. 
\end{align*}
Then, we have by Eq.~\eqref{eq: NNGP approximation} that 
\begin{align*}
    S_2 \leq C_1(z) \left( \sum_{i=0}^{T_0 } a_1^{2L_0 + T_0  +  i}\right).
\end{align*}
Therefore, combining the bounds on $S_1, S_2$, we conclude that 
\begin{align*}
  K^{(L_0 + T_0 + 1)}(z) - k_{\epsilon}^{(L_0 + T_0 + 1)}(z) &\leq S_1 + S_2 \\
  &\leq C_1(z) \left( \sum_{i=0}^{T_0 } a_1^{2L_0 + T_0  +  i}\right) + C_1(z) \left( \sum_{i =0}^{T_0-1} (T_0 - i) a_1^{2L_0 + T_0 + i} \right) \\
  &\hspace{8mm} + C_2(z)   \left(\sum_{i=0}^{T_0-1} (L_0 + i + 2) a_1^{2L_0 + T_0 + 1 + i} \right) \\
  &= C_1(z) \left( \sum_{i =0}^{T_0} (T_0 + 1 - i) a_1^{2L_0 + T_0 + i} \right) + C_2(z)   \left(\sum_{i=0}^{T_0-1} (L_0 + i + 2) a_1^{2L_0 + T_0 + 1 + i} \right),
\end{align*}
which concludes the proof by induction and establishes Eq.~\eqref{eq: kernel difference}.  Next, Eq.~\eqref{eq: kernel difference} implies: 
\begin{align*}
    \frac{K^{(L_0 + T)}(z) - k_{\epsilon}^{(L_0 + T)}(z)}{a_1^{L_0 + T}(L_0 + T + 1)}  &\leq C_1(z) \left( \sum_{i =0}^{T-1} \frac{T-i}{T + L_0 + 1} a_1^{L_0 - 1 + i} \right) + C_2(z) \left(\sum_{i=0}^{T-2} \frac{L_0 + i + 2}{T + L_0 + 1} a_1^{L_0 + i} \right) \\
    &\leq \left( C_1(z) a_1^{L_0 - 1} + C_2(z) a_1^{L_0} \right) \frac{1} {1 - a_1} ~. 
\end{align*}
Hence, since the right-hand side does not depend on $T$, we conclude that 
\begin{align*}
    \lim\limits_{L \to \infty} \frac{K^{(L)}(z) - k_{\epsilon}^{(L)}(z)}{a_1^{L}(L + 1)} \leq \left( C_1(z) a_1^{L_0 - 1} + C_2(z) a_1^{L_0} \right) \frac{1} {1 - a_1}.
\end{align*}
Lastly, note that by selecting $\epsilon$ small enough, we can make $L_0(z)$ arbitrarily large.  Hence, for any fixed $z \in [0, 1]$, we conclude that
\begin{align*}
    \lim_{\epsilon \to 0} \lim\limits_{L \to \infty}  \frac{K^{(L)}(z) - k_{\epsilon}^{(L)}(z)}{a_1^{L}(L+1)} = 0,
\end{align*}
and as a consequence that
\begin{align}
\label{eq: Incomplete upper bound on normalized iterated NTK}
      \lim\limits_{L \to \infty}  \frac{K^{(L)}(z)}{a_1^{L}(L+1)} &= \lim_{\epsilon \to 0} \lim\limits_{L \to \infty}  \frac{K^{(L)}(z) - k_{\epsilon}^{(L)}(z)}{a_1^{L}(L+1)} +  \lim_{\epsilon \to 0} \lim\limits_{L \to \infty} \frac{k_{\epsilon}^{(L)}(z)}{a_1^{L}(L+1)}  \nonumber \\
      &=  \lim_{\epsilon \to 0} \lim\limits_{L \to \infty} \frac{k_{\epsilon}^{(L)}(z)}{a_1^{L}(L+1)}.
\end{align}
By uniformly bounding the right-hand side over $\epsilon$, we will establish an upper bound on the pole order for the iterated, normalized NTK.  To do this, we first show that the iterated, normalized $k_{\epsilon}$ and $f_{\epsilon}$ are equal for $z \in (\epsilon, 1)$.  Let $\alpha(z) = k_{\epsilon}^{(L_0(z))}(z)$ and $\beta(z) = f_{\epsilon}^{(L_0(z))}(z)$ for $z \in (\epsilon, 1)$.  We prove by induction for $T > 0$ that
\begin{align}
\label{eq: Induction for surrogate kernel}
    k_{\epsilon}^{(L_0(z) + T)}(z) = a_1^{T} [\alpha(z) + T \beta(z)].  
\end{align}
The base case for $T = 1$ follows by
\begin{align*}
    k_{\epsilon}^{(L_0(z) + 1)}(z) = a_1 [\alpha(z)]  + a_1 \beta(z).
\end{align*}
Proceeding inductively, we assume that Eq.~\eqref{eq: Induction for surrogate kernel} holds for time $T$.  Then at time $T + 1$, we have 
\begin{align*}
    k_{\epsilon}^{(L_0(z) + T  + 1)}(z) &= a_1 k_{\epsilon}^{L_0(z) + T} + f_{\epsilon}^{L_0 + T + 1}(z) \\
    &= a_1^{T+1} [\alpha(z) + T \beta(z)] + a_1^{T+1} \beta(z) \\
    &= a_1^{T+1} [\alpha(z) + (T+1) \beta(z)], 
\end{align*}
which concludes the proof by induction. Thus, we obtain that 
\begin{align*}
    \lim_{L \to \infty} \frac{k_{\epsilon}^{(L)}(z)}{a_1^{L}(L+1)} &= \lim_{L \to \infty} \frac{k_{\epsilon}^{(L-L_0(z))}(\alpha(z))}{a_1^{L - L_0(z)}(L - L_0(z)+1)} \left(\frac{L - L_0(z) +1}{L + 1}\right) a_1^{-L_0(z)} \\
    & =\lim_{T \to \infty} \frac{a_1^{T} [\alpha(z) + T \beta(z)]}{a_1^{T}(T+1)}  a_1^{-L_0(z)} \\
    &= \beta(z) a_1^{-L_0(z)} \\
    &= \lim_{L \to \infty} \frac{f_{\epsilon}^{(L)}(z)}{a_1^L}.
\end{align*}

Next, we will uniformly bound the iterated, normalized $f_{\epsilon}$.  In particular, since $\check{\phi} \geq f_{\epsilon}$ and the two functions have the same normalizing constant, we obtain
\begin{align*}
    \lim_{L \to \infty}  \frac{f_{\epsilon}^{(L)}(z)}{a_1^L} \leq   \lim_{L \to \infty}  \frac{\check{\phi}^{(L)}(z)}{a_1^L}.
\end{align*}
 Now, we have that for any $\hat{\epsilon}$, $\check{\phi}$ is upper bounded by the function: 
\begin{align*}
    \tilde{g}_{\hat{\epsilon}}(z) = \begin{cases} \check{\phi}(z) ~~ \text{if $z \in [0, c)$}  \\ 1 - (\check{\phi}'(1)-\hat{\epsilon})(1- z) ~~ \text{if $z \in [c, 1]$}  \end{cases},
\end{align*}
where $z = c$ is the intersection of the secant line $1 - (\check{\phi}'(1)-\hat{\epsilon})(1- z)$ and $\check{\phi}$.  We visualize $\tilde{g}_{\hat{\epsilon}}(z)$ in Fig.~\ref{fig: Bounding Curves}c.  By Lemma \ref{lemma: NNGP linear around 1}, we know that for $z \in (c, 1)$:
\begin{align*}
    \lim_{L \to \infty} \frac{\tilde{g}_{\hat{\epsilon}}^{(L)}(z)}{a_1^L} &= \frac{R_{\hat{\epsilon}}(z)}{(1- z)^{-\log_{\check{\phi}'(1) - \hat{\epsilon}} \phi'(0)}},
\end{align*} 
where $R_{\hat{\epsilon}}(z)$ is non-negative for $z \in [0, 1]$, bounded from above, and bounded away from $0$ around $z = 1$.  Since $\hat{\epsilon}$ is arbitrary, we conclude that for some $\epsilon''$, for $z \in (\epsilon'', 1)$: 
\begin{align}
\label{eq: Upper bound on normalized iterated NNGP}
    \lim_{L \to \infty}  \frac{f_{\epsilon}^{(L)}(z)}{a_1^L} \leq   \lim_{L \to \infty}  \frac{\check{\phi}^{(L)}(z)}{a_1^L} \leq \frac{R_1(z)}{(1 - z)^{-\log_{\check{\phi}'(1)} \check{\phi}'(0)}},
\end{align}
where $R_1(z)$ is non-negative for $z \in [0, 1]$, bounded from above, and bounded away from $0$ around $z = 1$.  By substituting back the above inequalities into Eq.~\eqref{eq: Incomplete upper bound on normalized iterated NTK}, we conclude that
\begin{align}
\label{eq: Upper bound on normalized iterated NTK}
    \lim\limits_{L \to \infty}  \frac{K^{(L)}(z)}{a_1^{L}(L+1)} &= \lim_{\epsilon \to 0} \lim\limits_{L \to \infty}  \frac{K^{(L)}(z) - k_{\epsilon}^{(L)}(z)}{a_1^{L}(L+1)} +  \lim_{\epsilon \to 0} \lim\limits_{L \to \infty} \frac{k_{\epsilon}^{(L)}(z)}{a_1^{L}(L+1)} \leq \frac{R_1(z)}{(1 - z)^{ -\log_{\check{\phi}'(1)} \check{\phi}'(0)}}.
\end{align}

To conclude the proof, we need to establish a similar lower bound on the above limit.  We will construct the lower bound by first establishing the order of the singularity of the iteration of $\check{\phi}$ and then showing that this order is a lower bound on the order of the singularity for the iterated, normalized NTK.  Note that we have already established an upper bound on the order of the singularity of the iteration of $\check{\phi}$ in Eq.~\eqref{eq: Upper bound on normalized iterated NNGP}.  Now, we alternatively lower bound $\check{\phi}$ via the following function: 
\begin{align*}
    \tilde{f}(z) = \begin{cases} a_1 z ~~ \text{if $x \in [0, c)$}  \\ 1 - \check{\phi}'(1)(1- z) ~~ \text{if $x \in [c, 1]$}  \end{cases},
\end{align*}
where $z = c$ corresponds to the intersection of the tangent lines of $\check{\phi}$ at $z = 0$ and $z = 1$.  We visualize $\tilde{f}(z)$ in Fig.~\ref{fig: Bounding Curves}d.  By Lemma \ref{lemma: NNGP linear around 1}, we have that for $z \in (c, 1)$:
\begin{align*}
    \lim\limits_{L \to \infty}  \frac{\check{\phi}^{(L)}(z)}{a_1^{L}} \geq \lim\limits_{L \to \infty}  \frac{\tilde{f}^{(L)}(z)}{a_1^{L}} = \frac{Q(z)}{(1 - z)^{-\log_{\check{\phi}'(1)} \check{\phi}'(0)}},
\end{align*}
where $Q(z)$ is non-negative for $z \in [0, 1]$, bounded from above, and bounded away from $0$ around $z = 1$. Hence, we conclude that: 
\begin{align*}
    \lim\limits_{L \to \infty}  \frac{\check{\phi}^{(L)}(z)}{a_1^{L}} = \frac{R_2(z)}{(1 - z)^{-\log_{\check{\phi}'(1)} \check{\phi}'(0)}}, 
\end{align*}
where $R_2(z)$ is non-negative for $z \in [0, 1]$, bounded from above, and bounded away from $0$ around $z = 1$.  Lastly, we utilize Eq.~\eqref{eq: Closed form for NTK} to show that: 
\begin{align*}
    \lim\limits_{L \to \infty}  \frac{\check{\phi}^{(L)}(z)}{a_1^{L}} \leq  \lim\limits_{L \to \infty}  \frac{K^{(L)}(z)}{a_1^{L}(L+1)}.
\end{align*}
In particular, Eq.~\eqref{eq: Closed form for NTK} states that 
\begin{align*}
    K^{(L)}(z)= \sum_{i=0}^{L} \check{\phi}^{(i)}(z) \prod_{k=i}^{L-1} \check{\phi}'\left(\check{\phi}^{(k)}(z)\right).
\end{align*}
We next write $\check{\phi}^{(i)}(z)$ as a product and substitute the computed product back into Eq.~\eqref{eq: Closed form for NTK}.  Namely, using the power series representation for $\check{\phi}$ and unrolling the iteration, we obtain: 
\begin{align*}
    \check{\phi}^{(i)}(z) &= \sum_{j=1}^{\infty} a_j \left(\check{\phi}^{(i-1)}(z)\right)^{j} \\
    &= a_1 \check{\phi}^{(i-1)}(z) \left (1 + \sum_{j=2}^{\infty} \frac{a_j}{a_1} \left(\check{\phi}^{(i-1)}(z)\right)^{j} \right) \\
    &= a_1^{i} z \prod_{k=0}^{i - 1} \left( 1 + \sum_{j=2}^{\infty} \frac{a_j}{a_1} \left(\check{\phi}^{(k)}(z)\right)^{j} \right).
\end{align*}
We similarly use the power series for $\check{\phi}'(z)$ to conclude that
\begin{align*}
    \check{\phi}'\left(\check{\phi}^{(k)}(z)\right) &= \sum_{j=1}^{\infty} j a_j \left(\check{\phi}^{(k)}(z)\right)^{j-1} \\
    &\geq \sum_{j=1}^{\infty} a_j \left(\check{\phi}^{(k)}(z)\right)^{j-1} \\
    &\geq a_1 \left(1 + \sum_{j=2}^{\infty} \frac{a_j}{a_1} \left(\check{\phi}^{(k)}(z)\right)^{j-1} \right) \\
    &\geq a_1 \left(1 + \sum_{j=2}^{\infty} \frac{a_j}{a_1} \left(\check{\phi}^{(k)}(z)\right)^{j} \right) ~~ \left(\text{as $\check{\phi}^{(k)}(z) \leq 1$.} \right). 
\end{align*}
Therefore, we can simplify Eq.~\eqref{eq: Closed form for NTK} as follows: 
\begin{align*}
      K^{(L)}(z) &= \sum_{i=0}^{L} \check{\phi}^{(i)}(z) \prod_{k=i}^{L-1} \check{\phi}'\left(\check{\phi}^{(k)}(z)\right) \\
      &\geq \sum_{i=0}^{L} a_1^{i} z \prod_{k=0}^{i - 1} \left( 1 + \sum_{j=2}^{\infty} \frac{a_j}{a_1} \left(\check{\phi}^{(k)}(z)\right)^{j} \right) \prod_{k'=i}^{L-1} a_1 \left(1 + \sum_{j=2}^{\infty} \frac{a_j}{a_1} \left(\check{\phi}^{(k')}(z)\right)^{j} \right) \\
      &= \sum_{i=0}^{L} a_1^{L} z \prod_{k=0}^{L - 1} \left( 1 + \sum_{j=2}^{\infty} \frac{a_j}{a_1} \left(\check{\phi}^{(k)}(z)\right)^{j} \right) \\
      &= (L+1) \check{\phi}^{(L)}(z),
\end{align*}
and conclude that
\begin{align*}
    \frac{K^{(L)}(z)}{a_1^{L} (L + 1)} \geq \frac{\check{\phi}^{(L)}(z)}{a_1^{L}}. 
\end{align*}
As a consequence,
\begin{align}
\label{eq: Lower bound on normalized iterated NTK}
    \lim\limits_{L \to \infty}  \frac{K^{(L)}(z)}{a_1^{L}(L+1)} \geq \lim\limits_{L \to \infty}  \frac{\check{\phi}^{(L)}(z)}{a_1^{L}} = \frac{R_2(z)}{(1 - z)^{-\log_{\check{\phi}'(1)} \check{\phi}'(0)}}.
\end{align}
Lastly, we combine Eq.~\eqref{eq: Upper bound on normalized iterated NTK} and \eqref{eq: Lower bound on normalized iterated NTK} to conclude that there exists some $\epsilon$ such that for $z \in (\epsilon, 1)$: 
\begin{align*}
   \frac{R_2(z)}{(1 - z)^{-\log_{\check{\phi}'(1)} \check{\phi}'(0)}} \leq \lim\limits_{L \to \infty}  \frac{K^{(L)}(z)}{a_1^{L}(L+1)} \leq \frac{R_1(z)}{(1 - z)^{-\log_{\check{\phi}'(1)} \check{\phi}'(0)}}~. 
\end{align*}
Thus, we conclude that
\begin{align*}
    \lim\limits_{L \to \infty}  \frac{K^{(L)}(z)}{a_1^{L}(L+1)} = \frac{R(z)}{(1 - z)^{-\log_{\check{\phi}'(1)} \check{\phi}'(0)}}, 
\end{align*}
where $R(z)$ is non-negative for $z \in [0, 1]$, bounded from above, and bounded away from $0$ around $z = 1$.  This concludes the proof of Theorem \ref{theorem: NTK singular kernel}.
\end{proof}


To prove Theorem \ref{theorem: Optimality of the NTK}, we will use the result of Theorem \ref{theorem: NTK singular kernel} and that of~\cite{DevroyeHilbertKernel}, which analyzes the optimality of singular kernel smoothers.  To connect infinitely wide and deep networks with kernel smoothers, we next prove that the infinite depth limit of the NTK corresponds to a kernel smoother under the conditions of Theorem \ref{theorem: NTK singular kernel}.


\begin{lemma}
\label{lemma: convergence to kernel smoother}
Let $\psi(z) = \lim\limits_{L \to \infty} \frac{K^{(L)}(z)}{a_1^{L} (L+1)}$.  Then under the setting of Theorem \ref{theorem: Optimality of the NTK},  
\begin{align*}
m_n(x) = \sign \left( \sum\limits_{i=1}^{n} y^{(i)} \psi(x^T x^{(i)})\right),    
\end{align*}
assuming $\left|\sum\limits_{i=1}^{n} y^{(i)} \psi(x^T x^{(i)})\right| > 0$ for almost all $x \in \mathcal{S}_+^{d-1}$.  
\end{lemma}

\begin{proof}
Let $m_n^{(L)}(x)$ be defined as follows: 
\begin{align*}
    m_n^{(L)}(x) = \sign\left( y {\left(K_n^{{(L)}}\right)^{-1}} K^{(L)}(X, x)\right).
\end{align*}
We first note that multiplying the argument to the $\sign$ function by a positive constant does not affect the value.  Hence, we have: 
\begin{align*}
    \lim\limits_{L \to \infty} m_n^{(L)}(x) = \lim\limits_{L \to \infty}  \sign\left(y {\left(K_n^{(L)}\right)}^{-1} \frac{K^{(L)}(X, x)}{a_1^L (L+1)} \right). 
\end{align*}
Now we compare the argument of the $\sign$ function above to the corresponding kernel smoother.  Namely,  we have: 
\begin{align*}
\left| y {\left(K_n^{(L)}\right)}^{-1} \frac{K^{(L)}(X, x)}{a_1^L (L+1)}  - y\frac{K^{(L)}(X, x)}{a_1^L (L+1)} \right| \leq  \|y\|_2 \left\|{\left(K_n^{(L)}\right)}^{-1} - I\right\|_2 \left\| \frac{K^{(L)}(X, x)}{a_1^L (L+1)}\right\|_2,
\end{align*}
where the inequality follows from the Cauchy-Schwarz inequality and $\|Av\|_2 \leq \|A\|_2 \|v\|_2$ for $A \in \mathbb{R}^{n \times n}, v \in \mathbb{R}^{n}$.  Now since $0$ is an attractor for $\check{\phi}$, then for any $h > 0$, there exists $L_1$ such that for $L > L_1$, the spectrum of $\hat{K}^{(L)}$ is contained in $[1 - hn^2, 1 + hn^2]$ by Weyl's inequalities.  Hence, the spectrum of ${\left(K_n^{(L)}\right)}^{-1}$ is contained in  $\left[\frac{1}{1 + hn^2}, \frac{1}{1 - hn^2}\right]$.  Thus, we conclude that
\begin{align*}
    \left\|{\left(K_n^{(L)}\right)}^{-1} - I\right\|_2 \leq \left(\frac{1}{1 - hn^2} - 1\right).
\end{align*}
Hence by selecting $h$ appropriately small, we conclude that for any $\epsilon_1$, there exists $L_1$ such that for $L > L_1$,  $\left\|{\left(K_n^{(L)}\right)}^{-1} - I\right\|_2  < \epsilon_1$.  Next, since $\lim\limits_{L \to \infty} \frac{K^{(L)}(x^{(i)}, x)}{a^L (L+1)} = \psi(x^Tx^{(i)})$, for any $\epsilon_2$, we can select $L_2$ such that for $L > L_2$, 
$$ \left| y\frac{K^{(L)}(X, x)}{a^L (L+1)} - \sum_{i=1}^{n} y^{(i)}\psi(x^T x^{(i)}) \right| < \epsilon_2.$$
Next under the assumption in the lemma, we may thus select $\epsilon_1, \epsilon_2$ small enough such that the argument of $m_n^{(L)}(x)$ is not exactly $0$ for $L > \max(L_1, L_2)$. Thus we can interchange the limit and the $\sign$ function. As a consequence, for any $x \neq x^{(i)}$ for $i \in \{1, 2, \ldots, n\}$ satisfying $\sum_{i=1}^{n} y^{(i)} \psi(x^T x^{(i)}) \neq 0$, we obtain that 
\begin{align*}
    \lim\limits_{L \to \infty} m_n^{(L)}(x) &= \lim\limits_{L \to \infty}  \sign\left(y {\left(K_n^{(L)}\right)}^{-1} \frac{K^{(L)}(X, x)}{a^L (L+1)} \right) \\
    &=  \sign\left( \lim\limits_{L \to \infty}  y {\left(K_n^{(L)}\right)}^{-1} \frac{K^{(L)}(X, x)}{a^L (L+1)} \right) \\
    &= \sign\left( \sum_{i=1}^{n} y^{(i)} \psi(x^T x^{(i)}) \right).
\end{align*}
Lastly, if $x = x^{(i)}$ for some $i \in \{1, 2, \ldots, n\}$, then since $\psi(z)$ has a singularity at $z = 1$,
$$\lim_{L \to \infty} m_n^{(L)}(x) = \sign\left(\lim\limits_{z \to 1} \sum_{i=1}^{n} \frac{1}{\psi(z)} y^{(i)} \psi(x^T x^{(i)}) \right) = \sign(y^{(i)}),$$
which completes the proof.
\end{proof}

We lastly utilize Theorem \ref{theorem: NTK singular kernel}, Lemma \ref{lemma: convergence to kernel smoother}, and the result of~\cite{DevroyeHilbertKernel} to prove Theorem \ref{theorem: Optimality of the NTK} (expressed below in terms of dual activations), which identifies infinitely wide and deep classifiers that achieve optimality.  

\begin{theorem*}
Let $m_n$ denote the classifier in Eq.~\eqref{eq: NTK Classifier} corresponding to training an infinitely wide and deep network with activation function $\phi$ on $n$ training points. Let $m$ denote the Bayes optimal classifier, i.e. $m(x) = \argmax\limits_{\tilde{y} \in \{-1, 1\}} \mathbb{P}\left( y = \tilde{y} | x \right)$.  If the dual activation, $\check{\phi}$ satisfies: 
\begin{enumerate}
\item[1)] $\check{\phi}(0) = 0$, $\check{\phi}(1) = 1$,
\item[2)] $0 < \check{\phi}'(0) < 1$ and $\check{\phi}'(1) < \infty$,
\item[3)] $-\frac{\log(\check{\phi}'(0))}{\log\left(\check{\phi}'(1)\right)} = \frac{d}{2}$,
\end{enumerate}
then $m_n$ satisfies $\lim\limits_{n \to \infty} \mathbb{P}_X\left( \left| m_n(x) -  m(x) \right| > \epsilon \right) = 0$ for almost all $x \in \mathcal{S}_+^{d}$ and for any $\epsilon > 0$.

\end{theorem*}

\begin{proof}
Thus far, we proved that under the conditions of Theorem \ref{theorem: NTK singular kernel}, the classifier $m_n$ corresponds to taking the sign of a kernel smoother using a singular kernel with singularity of order $- \frac{\log(\check{\phi}'(0))}{\log(\check{\phi}'(1))}$.  For data with density in $\mathbb{R}^{d}$, kernel smoothers with singular kernels of the form $K_h(x, \tilde{x}) = \frac{1}{\|x - \tilde{x}\|^d}$ (i.e., the Hilbert estimate) converge to the Bayes optimal classifier in probability for almost all samples as $n \to \infty$~\cite{DevroyeHilbertKernel}.   We note that multiplying $K_h(x, \tilde{x})$ by a non-negative function that is bounded away from $0$ around $1$ and bounded from above such that the kernel is still monotonically increasing also yields optimality in the same sense (see Supplementary Information \ref{appendix: C}).  Returning to our setting, for any $x, \tilde{x} \in \mathcal{S}_+^{d}$, we can re-write the kernel $K_h(x, \tilde{x})$ as
\begin{align*}
    K_h(x, \tilde{x}) = \frac{1}{\|x - \tilde{x}\|^d} = \frac{1}{2^{\frac{d}{2}}(1 - x^T \tilde{x})^{\frac{d}{2}}}.
\end{align*}
The constant $\frac{1}{2^{\frac{d}{2}}}$ again does not affect the $\sign$ function.  Lastly, assumption 3 selects the order of the singularity such that the limiting kernel from Theorem \ref{theorem: NTK singular kernel} can be written up to constant factors as a Hilbert estimate, which concludes the proof of Theorem \ref{theorem: Optimality of the NTK}.  
\end{proof}

\section{Extension of Hilbert estimate optimality from~\cite{DevroyeHilbertKernel}}
\label{appendix: C}

We utilize the following extension of the result from~\cite{DevroyeHilbertKernel} to prove Theorem \ref{theorem: Optimality of the NTK}.  In this section, we follow the notation from \cite{DevroyeHilbertKernel} in our statements and proofs.   

\begin{corollary*}
\label{theorem: Devroye optimality extension}
For $x \in \mathcal{S}_+^{d}$, let $m(x)$ denote the Bayes optimal regressor.  For $x, \tilde{x} \in \mathcal{S}_+^{d}$,  let $K(x^T \tilde{x}) = \frac{R\left( x^T \tilde{x} \right)}{2^{\frac{d}{2}} (1 - x^T \tilde{x})^{\frac{d}{2}}}$, where $R(z) \geq 0$ for $z \in [0, 1]$ is bounded from above, bounded away from $0$ around $z = 1$, and $K(\cdot)$ is monotonically increasing in $[0, 1]$.  Given a dataset $\{X_i, Y_i\}_{i=1}^{n} \subset \mathcal{S}_+^{d} \times \mathbb{R}$, let
\begin{align*}
    m_n(x) = \frac{\sum_{i=1}^{n} Y_i K(x^TX_i)}{\sum_{i=1}^{n}K(x^TX_i)}.
\end{align*}
Let $X$ have any density $f$ on $\mathcal{S}_+^{d}$ and let $Y$ be bounded.  Then, at almost all $x$ with $f(x) > 0$, $m_n(x) \to m(x)$ in probability as $n \to \infty$.  
\end{corollary*}


\begin{proof}
The proof closely follows that of the theorem in~\cite{DevroyeHilbertKernel} with the differences that (1) we map from densities on $\mathcal{S}_+^{d}$ to densities on $\mathbb{R}^{d}$, and (2) we simply verify that the function $R(z)$ does not change the asymptotic analyses of the original proof.  We begin by noting that the kernel $K$ involves chordal distances on the sphere, i.e.,
\begin{align*}
    K(x^T \tilde{x}) = \frac{R\left( x^T \tilde{x} \right)}{\|x - \tilde{x}\|^d}.
\end{align*}
We first define the random variable $W := \|P(x) - P(X)\|^d V_d$, where $V_d$ is the volume of the unit sphere in $d$ dimensions and $P: \mathcal{S}^{d} \to \mathbb{R}^{d}$ is the stereographic projection such that $\mathcal{S}_+^{d}$ maps to a bounded region.    We let $f_P$ denote the density of the points $P(x)$ for $x \in \mathcal{S}_+^d$.  We note that Euclidean distances after stereographic projection can be related to chordal distances, $\|x - X\|$, via the following formula (up to isometries of the sphere): 
\begin{align*}
    \|x - X\|^2 = \frac{\|P(x) - P(X)\|^2}{(1 +  \|P(x)\|^2)(1 + \|P(X)\|^2)}.
\end{align*}
Since we select the projection such that $\|P(x)\| < \infty$ for $x \in \mathcal{S}_+^{d}$, we have that $(1 +  \|P(x)\|^2)(1 + \|P(X)\|^2)$ is bounded and nonzero, i.e., it is again a factor that simply scales the kernel function.  We thus define
\begin{align*}
    Q(x, \tilde{x}) = R(x^T \tilde{x}) (1 + \|P(x)\|^2)^{\frac{d}{2}} (1 + \|P(\tilde{x})\|^2)^{\frac{d}{2}}, 
\end{align*}
which is bounded away from zero for some $\epsilon > 0$ and $x, \tilde{x}$ such that $x^T \tilde{x} > 1 - \epsilon$.  Letting $W_i := V_d\|P(x) - P(X_i)\|^{d}$, the regressor $m_n(x)$ is given by 
\begin{align*}
    m_n(x) = \frac{\sum_{i=1}^{n} Y_i \frac{Q(x, X_i)}{W_i}}{\sum_{i=1}^{n}\frac{Q(x, X_i)}{W_i}}. 
\end{align*}
Hence, we can utilize the proof strategy of~\cite{DevroyeHilbertKernel} for points $P(x)$ in $\mathbb{R}^{d}$.  Namely as in~\cite{DevroyeHilbertKernel}, we analyze the term:  
\begin{align*}
    |m_n(x) - m(x)| \leq \left|\frac{\sum_{i=1}^{n} (Y_i - m(X_i)) \frac{Q(x, X_i)}{W_i}}{\sum_{i=1}^{n} \frac{Q(x, X_i)}{W_i}} \right| + \frac{\sum_{i=1}^{n} |m(X_i) - m(X))| \frac{Q(x,  X_i)}{W_i}}{\sum_{i=1}^{n} \frac{Q(x, X_i)}{W_i}} := I + II.
\end{align*}
To simplify notation, we let $Q_i = Q(x, X_i)$ and we let $\frac{Q_{(i)}}{W_{(i)}}$ denote the $i$\textsuperscript{th} order statistic ordered such that  $W_{(1)} \leq W_{(2)} \ldots \leq W_{(n)}$.  Now, the proof strategy of~\cite{DevroyeHilbertKernel} is to show that the terms $I$ and $II$ respectively converge to $0$ in probability for almost all $x$ as $n \to \infty$.  To prove that $I$ converges in this manner, following the proof of~\cite{DevroyeHilbertKernel}, we have that: 
\begin{align*}
    \mathbb{E}[I^2 | \{X_i\}_{i=1}^n] \leq C_1 \frac{\frac{1}{W_{(1)}}}{\sum_{j=1}^{k} \frac{Q_{(j)}}{W_{(j)}}} \leq C_2 \frac{\frac{1}{W_{(1)}}}{\sum_{j=1}^{k} \frac{1}{W_{(j)}}},
\end{align*}
where $k$ such that $W_{(k)} > V_d \delta^d$ for small $\delta$, and $C_1, C_2 > 0$ are constants since $\{Q_{(j)}\}_{j=1}^{k}$ are non-negative and bounded away from $0$.  Hence, the convergence of $I$ follows directly from the proof of~\cite{DevroyeHilbertKernel}.    To establish the convergence of $II$, we follow the proof of~\cite{DevroyeHilbertKernel} and first establish that
\begin{align*}
    A_n := \frac{\sum_{i \leq \theta n} \frac{Q_{(i)}}{W_{(i)}}}{\sum_{i=1}^{n} \frac{Q_{(i)}}{W_{(i)}}} \to 1 
\end{align*}
in probability as $n \to \infty$, for all $\theta$ fixed in $(0, 1)$.  Let $\chi$ denote the indicator function, and following the notation of~\cite{DevroyeHilbertKernel}, let $U_{(i)}$ denote uniform order statistics.  The work of~\cite{DevroyeHilbertKernel} establishes that for any fixed $\epsilon \in (0, 1)$ there exists $\delta$ such that for all $W_{(i)} \leq V_d \delta^{d}$: 
\begin{align*}
    (1 - \epsilon)f_P(P(x)) W_{(i)} \leq U_{(i)} \leq (1 + \epsilon)f_P(P(x)) W_{(i)}.
\end{align*}
Hence, we consider the event $B =[W_{\lfloor \theta n \rfloor} \leq V_d \delta^d ]$, and then as in the proof of~\cite{DevroyeHilbertKernel}, we obtain
\begin{align*}
    A_n \chi_{B} \geq 1 - \frac{2\epsilon}{1 + \epsilon} - \frac{\frac{n C_3 }{W_{\lfloor \theta n\rfloor}}}{f_P(P(x)) \sum_{i \leq \theta n} \frac{Q_{(i)}}{U_{(i)}}} \geq 1 - \frac{2\epsilon}{1 + \epsilon} - C_4 \frac{\frac{n}{W_{\lfloor \theta n\rfloor}}}{f_P(P(x)) \sum_{i \leq \theta n} \frac{1}{U_{(i)}}}, 
\end{align*}
where $C_3, C_4 > 0$ are constants since $Q_{(i)}$ is bounded and positive for $i \leq \lfloor \theta n \rfloor$.  The convergence of $A_n$ then follows by continuing the proof from~\cite{DevroyeHilbertKernel}.  Next, again following the proof of~\cite{DevroyeHilbertKernel}, for any $\epsilon > 0$, we also select $\delta$ such that: 
\begin{align*}
    \sup_{r \leq \delta} \frac{\int_{S_{P(x), r}} | m(y) - m(x) | f_P(y) dy}{\int_{S_{P(x), r}}  f_P(y) dy} \leq \epsilon,
\end{align*}
where $S_{P(x), r}$ denotes the closed ball in $\mathbb{R}^{d}$ of radius $r$ centered at $P(x)$.  Then as in~\cite{DevroyeHilbertKernel}, select $A = \{y : m(y) - m(x) > \epsilon\}$ and select $\theta \in (0, 1)$ small enough such that $\mathbb{P}(\|P(X_{(\lfloor \theta n \rfloor)}) - P(x)\| > \delta)  \to 0$ as $n \to \infty$.  Then, we have: 
\begin{align*}
    II &= \frac{\sum_{i=1}^{n} |m(X_i) - m(x))| \frac{Q_i}{W_i}}{\sum_{i=1}^{n} \frac{Q_i}{W_i}} \\
    &\leq 2 \frac{\sum_{i > \theta n} \frac{Q_i}{W_i}}{\sum_{i=1}^n \frac{Q_i}{W_i}} + 2 \chi_{\|P(X_{\lfloor \theta n \rfloor}) - P(x)\| > \delta} + \epsilon + \frac{\sum_{i : P(X_i) \in S_{P(x), \delta} \cap A} \frac{Q_i}{W_i}}{\sum_{i=1}^n \frac{Q_i}{W_i}} \\
    &:= V_1 + V_2 + V_3 + V_4.
\end{align*}
  Now as in~\cite{DevroyeHilbertKernel}, we have that $V_1 \to 0$ in probability, as we showed $A_n \to 1$ in probability above. Then, $V_2 \to 0$ in probability and $V_3$ can be made as small as possible by the choice of $\epsilon$.  Lastly, $V_4 \to 0$ since, following the proof of~\cite{DevroyeHilbertKernel}: 
\begin{align*}
    \frac{\sum_{i : P(X_i) \in S_{P(x), \delta} \cap A} \frac{Q_i}{W_i}}{\sum_{i=1}^n \frac{Q_i}{W_i}} &\leq 2\epsilon + C_5 \frac{\frac{1}{W_{(1)}}}{\sum_{i=1}^{n}\frac{Q_{(i)}}{W_{(i)}}},
\end{align*}
where $C_5 > 0$ is a constant.  The above term goes to $0$ in probability by the analysis of part $I$ and the arbitrary choice of $\epsilon$.  This concludes the proof of this extension of the result of~\cite{DevroyeHilbertKernel}. 
\end{proof}

\section{Proof of Corollary \ref{corollary: example optimal NTK}}
\label{appendix: D}

For ease of reading, we repeat Corollary \ref{corollary: example optimal NTK} below. 

\begin{corollary*}
Let $m_n$ denote the classifier in Eq.~\eqref{eq: NTK Classifier} corresponding to training an infinitely wide and deep network with activation function  
\begin{align*}
    \phi(x) &= \begin{cases}
    \frac{1}{12 \sqrt{70}} h_7(x) + \frac{1}{\sqrt{2}} x   & \text{if $d = 1$} \\
    \frac{1}{2^{d/4}} \left(\frac{x^3 - 3x}{\sqrt{6}}\right) + \sqrt{1 - \frac{2}{2^{d/2}}} \left(\frac{x^2 - 1}{\sqrt{2}}\right) +  \frac{1}{2^{d/4}}x & \text{if $d \geq 2$}  \end{cases},
\end{align*} 
where $h_7(x)$ is the $7$\textsuperscript{th} probabilist's Hermite polynomial.\footnote{For $d = 1$, this activation function can be written in closed form as $\frac{x^7 - 21x^5 + 105x^3 + (12 \sqrt{35} -105)x}{12\sqrt{70}}$.} Then the classifier $m_n$ is Bayes optimal.
\end{corollary*}

\begin{proof}
We need only check that $\check{\phi}(z)$ satisfies the conditions of Theorem \ref{theorem: Optimality of the NTK}.  We first consider the case $d \geq 2$.  In particular, since $\frac{x^2 - 1}{\sqrt{2}}$ is the 2nd normalized probabilist's Hermite polynomial and $\frac{x^3 - 3x}{\sqrt{6}}$ is the third normalized probabilist's Hermite polynomial, we have by  \cite[Lemma 11]{DualActivation} that
\begin{align*}
    \check{\phi}(z) = \frac{1}{2^{\frac{d}{2}}} z^3 + \left(1 - \frac{2}{2^{\frac{d}{2}}}\right) z^2 +  \frac{1}{2^{\frac{d}{2}}} z.
\end{align*}
We thus have by direct computation that
\begin{align*}
    \check{\phi}'(1) &= \frac{3}{2^{\frac{d}{2}}} + \left(2 - \frac{4}{2^{\frac{d}{2}}} \right) + \frac{1}{2^{\frac{d}{2}}} = 2 ~~;~~
    \check{\phi}'(0) = \frac{1}{2^{\frac{d}{2}}}, 
\end{align*}
and so, the result follows from Theorem \ref{theorem: Optimality of the NTK} since 
\begin{align*}
    - \log_{\check{\phi}'(1)} \check{\phi}'(0) = \log_2 2^{\frac{d}{2}} = \frac{d}{2}.
\end{align*}
Now for the case of $d = 1$, we have again by \cite[Lemma 11]{DualActivation} that 
\begin{align*}
    \check{\phi}(z) = \frac{z^7}{2} + \frac{z}{2}.
\end{align*}
By direct computation,
\begin{align*}
    \check{\phi}'(1) &= \frac{7}{2} + \frac{1}{2} = 4~~ \textrm{and} ~~~
    \check{\phi}'(0) = \frac{1}{2}. 
\end{align*}
Hence, the result follows from Theorem \ref{theorem: Optimality of the NTK} since
\begin{align*}
    - \log_{\check{\phi}'(1)} \check{\phi}'(0) = \frac{1}{2} = \frac{d}{2}.
\end{align*}
\end{proof}

\section{Proof of Theorem \ref{theorem: NTK 1NN}}
\label{appendix: E}

We repeat Theorem \ref{theorem: NTK 1NN} below in terms of dual activations. 

\begin{theorem*}
Let $m_n$ denote the classifier in Eq.~\eqref{eq: NTK Classifier} corresponding to training an infinitely wide and deep network with activation function $\phi(\cdot)$ on $n$ training examples.  If the dual activation, $\check{\phi}$, satisfies: 
\begin{enumerate}
\item[1)] $\check{\phi}(0) = 0$, $\check{\phi}(1) = 1$,
\item[2)] $\check{\phi}'(0) = 0$, $\check{\phi}'(1) < \infty$~,
\end{enumerate}
then $m_n(x)$ is the 1-NN classifier for $x \in \mathcal{S}_+^{d}$.
\end{theorem*}

\begin{proof}
Let $m_n^{(L)}(x)$ be defined as follows: 
\begin{align*}
    m_n^{(L)}(x) = \sign\left( y {\left(K_n^{{(L)}}\right)^{-1}} K^{(L)}(X, x)\right).
\end{align*}
By the proof of Lemma \ref{lemma: convergence to kernel smoother}, we analogously have that the Gram matrix converges to the identity matrix as depth approaches infinity, i.e. $\lim_{L \to \infty} K_n^{(L)} = I$.  For $x, \tilde{x} \in \mathcal{S}_+^{d}$, let $z = x^T \tilde{x}$ and consider the radial kernel $K^{(L)}(z) = K^{(L)}(x, \tilde{x})$.   Let $\check{\phi}(z) = \sum_{i=2}^{\infty} a_i z^i$ for $a_i \geq 0$, as given by Eq.~\eqref{eq: Analytic formula dual activation}.  Without loss of generality, we assume $a_2 > 0$.  The proof will follow by using induction to establish: 
\begin{align}
\label{eq: 1NN kernel closed form} 
    \check{\phi}^{(L)}(z) = z^{2^L} h_L(z) ~~ \text{and} ~~ K^{(L)}(z) = z^{2^L} g_L(z), 
\end{align}
where $h_L, g_L$ are positive, increasing functions on $(0, 1]$.  The base case follows for $L = 0$ since $\check{\phi}^{(0)}(z) = K^{(0)}(z) = z$.   Hence, we assume the statement is true for $L = T - 1$ and prove the statement for $L = T$. We have 
\begin{align*}
    \check{\phi}^{(T)}(z) &= \check{\phi}\left( \check{\phi}^{(T-1)}(z) \right) = \sum_{i=2}^{\infty} a_i \left(\check{\phi}^{(T-1)}(z)\right)^{i} = \left(\check{\phi}^{(T-1)}(z)\right)^2 \left[\sum_{i=2}^{\infty} a_i \left(\check{\phi}^{(T-1)}(z)\right)^{i-2} \right],
\end{align*}
and hence using the inductive hypothesis, we can conclude that 
\begin{align*}
    \check{\phi}^{(T)}(z) = z^{2^T} h_{T-1}(z)^2 \left[\sum_{i=2}^{\infty} a_i \left(\check{\phi}^{(T-1)}(z)\right)^{i-2} \right] = z^{2^T} h_T(z), 
\end{align*}
where $h_T$ is positive and increasing since $h_{T-1}$ and the term in brackets are positive and increasing.  We proceed similarly for $K^{(T)}$. Namely, we have: 
\begin{align*}
    K^{(T)}(z) &= K^{(T-1)}(z) \check{\phi}'(\check{\phi}^{(T-1)}(z)) + \check{\phi}^{(T)}(z) \\
    &= z^{2^{T-1}} g_{T-1}(z) \left[ \sum_{i=2} i a_i \left( \check{\phi}^{(T-1)}(z)\right)^{i-1} \right] + z^{2^T} h_T(z) \\
    &= z^{2^{T-1}} \check{\phi}^{(T-1)}(z) g_{T-1}(z) \left[ \sum_{i=2} i a_i \left( \check{\phi}^{(T-1)}(z)\right)^{i-2} \right] + z^{2^T} h_T(z) \\
    &= z^{2^T} \left( h_{T-1}(z) g_{T-1}(z)\left[ \sum_{i=2} i a_i \left( \check{\phi}^{(T-1)}(z)\right)^{i-2} \right] + h_T(z)  \right) \\
    &= z^{2^T} g_T(z), 
\end{align*}
where $g_T(z)$ is positive and increasing since $h_T, h_{T-1}, g_{T-1}$ and the term in brackets are positive and increasing, which completes the induction argument. 

Now let $z_i = x^T x^{(i)}$ for $i \in \{1, 2, \ldots, n\}$.  Without loss of generality assume that  $z_1 > z_j$ for all $j \neq 1$.  To show that $\lim_{L \to \infty} m_n^{(L)}(x)$ is equivalent to the 1-NN classifier, we need only show that $\lim_{L \to \infty} m_n^{(L)}(x) = y^{(1)}$.  By Eq.~\eqref{eq: 1NN kernel closed form} for $j \neq 1$, we have that
\begin{align*}
    \lim_{L \to \infty} \frac{K^{(L)}(z_j)}{K^{(L)}(z_1)} &= \lim_{L \to \infty} \frac{  z_j^{2^{L}} g_L(z_j)}{z_1^{2^{L}} g_L(z_1)}  \\
    &\leq \lim_{L \to \infty} \frac{ z_j^{2^{L}}}{z_1^{2^{L}}} ~~ \left( \text{since $z_j < z_1$ and $g_L$ are positive and increasing} \right)\\
    &= 0.
\end{align*}
As a consequence, since $K^{(L)}(z_1) > 0$, we obtain that 
\begin{align*}
    \lim_{L \to \infty} m_n^{(L)}(x) =  \lim_{L \to \infty} \sign\left( y {\left(K_n^{{(L)}}\right)^{-1}} \frac{K^{(L)}(X, x)}{K^{(L)}(x^{(1)}, x)}\right) = y^{(1)}, 
\end{align*}
which establishes that $\lim_{L \to \infty} m_n^{(L)}(x)$ converges to the 1-NN classifier, thereby completing the proof.  
\end{proof}

\section{Proof of Proposition \ref{prop: NTK Majority Vote Condition}}
\label{appendix: F}

We repeat Proposition \ref{prop: NTK Majority Vote Condition} below for ease of reading.

\begin{prop*}
Let $m_n$ denote the classifier in Eq.~\eqref{eq: NTK Classifier} corresponding to training an infinitely wide and deep network with activation function $\phi(\cdot)$ on $n$ training examples. For $x, \tilde{x} \in \mathcal{S}_+^{d}$ with $x \neq \tilde{x}$, if the NTK $K^{(L)}$ satisfies 
\begin{align}
    \label{eq: NTK Majority vote condition}
    \lim_{L \to \infty} \frac{K^{(L)}(x, \tilde{x})}{C(L)} > C_1 \quad \textrm{and } \quad \lim_{L \to \infty} \frac{K^{(L)}(x, \tilde{x})}{C(L)} \neq \lim_{L \to \infty} \frac{K^{(L)}(x, x)}{C(L)}
\end{align}
with $C_1 > 0$ and $0 < C(L) < \infty$ for any $L$, then $m_n$ implements the majority vote classifier, i.e., 
\begin{align*}
    m_n(x) = \sign \Big(\sum_{i=1}^{n} y^{(i)} \Big)~.
\end{align*}
\end{prop*}

\begin{proof}
Let $C_2 =  \lim_{L \to \infty} \frac{K^{(L)}(x, x)}{C(L)}$.  We consider two cases: (1) when $ C_2 = \infty$, and (2) when $C_2 < \infty$.  When $C_2 = \infty$, we have: 
\begin{align*}
    \lim_{L \to \infty} m_n^{(L)}(x) &= \lim_{L \to \infty} \sign\left( y (K_n^{(L)})^{-1} K^{(L)}(X, x) \right) \\
    &= \lim_{L \to \infty} \sign\left( y \left( \frac{K_n^{(L)}}{K^{(L)}(x, x)} \right)^{-1} \frac{K^{(L)}(X, x)}{C(L)} \right) \\
    &= \sign\left( \sum_{i=1}^{n} y^{(i)} C_1 \right) \\
    &= \sign\left( \sum_{i=1}^{n} y^{(i)} \right), 
\end{align*}
which corresponds to the majority vote classifier.  When $C_2 < \infty$, we use the Sherman-Morrison formula to compute the inverse of the Gram matrix $\lim_{L \to \infty} (K_n^{(L)})^{-1}$.  In particular, since the inverse is a continuous map on invertible matrices, 
\begin{align*}
    \lim_{L \to \infty} (K_n^{(L)})^{-1} = \frac{1}{(C_2 - C_1)} I - \frac{C_1}{(C_2-C_1)(C_2 - C_1 + C_1n)} J,
\end{align*}
where $I$ is the identity matrix and $J$ is the all-ones matrix.  Hence, we have that for $x \neq x^{(i)}$ for $i \in \{1, 2, \ldots, n\}$: 
\begin{align*}
    \lim_{L \to \infty} y (K_n^{(L)})^{-1} \frac{K^{(L)}(X, x)}{C(L)} &= y \left(  \frac{1}{(C_2 - C_1)} I - \frac{C_1}{(C_2-C_1)(C_2 - C_1 + C_1n)} J \right) C_1 \mathbf{1} \\
    &= \frac{C_1}{C_2 - C_1 + C_1n} \sum_{i=1}^{n} y^{(i)},
\end{align*}
where $\mathbf{1} \in \mathbb{R}^{n}$ is the all-ones vector.  Assuming that $\sum_{i=1}^{n} y^{(i)} \neq 0$, we can swap the limit and $\sign$ function to conclude that: 
\begin{align*}
     \lim_{L \to \infty} m_n^{(L)}(x) &= \sign \left( \lim_{L \to \infty}  y (K_n^{(L)})^{-1} K^{(L)}(X, x) \right) \\
     &= \sign \left(  \frac{C_1}{C_2 - C_1 + C_1n} \sum_{i=1}^{n} y^{(i)} \right) \\
     &= \sign \left( \sum_{i=1}^{n} y^{(i)} \right),
\end{align*}
which completes the proof.
\end{proof}

\section{Proofs for when Infinitely Wide and Deep Networks are Majority Vote Classifiers}
\label{appendix: G}
The following lemma implies that any activation function satisfying $\check{\phi}(0) > 0$ and $\check{\phi}'(1) > 1$ yields a NTK satisfying Eq.~\eqref{eq: NTK Majority vote condition} and thus, the infinite depth classifier is the majority vote classifier by Proposition \ref{prop: NTK Majority Vote Condition}.

\begin{lemma}
\label{lemma: chaotic phase majority vote}
Let $m_n$ denote the classifier in Eq.~\eqref{eq: NTK Classifier} corresponding to training an infinitely wide and deep network with activation function $\phi(\cdot)$ on $n$ training examples.  If $\check{\phi}$ satisfies: 
\begin{enumerate}
    \item[1)] $\check{\phi}(0) > 0$, $\check{\phi}(1) = 1$,
    \item[2)] $1 < \check{\phi}'(1) < \infty$,
\end{enumerate}
then $m_n$ is the majority vote classifier. 
\end{lemma}

\begin{proof}
We show that the limiting kernel satisfies the properties of Proposition \ref{prop: NTK Majority Vote Condition} with $C_2 = \infty$.  Note that we must have $\check{\phi}'(0) < 1$ by Lemma \ref{lemma: 1st derivative of dual values}.  Now, since $\check{\phi}(0) < 1$ and $\check{\phi}(1) = 1$, by the intermediate value theorem, there exists some $c \in (0, 1)$ such that $\check{\phi}(c) = c$.  

We claim that $\check{\phi}'(c) < 1$.  Suppose for the sake of contradiction that $\check{\phi}'(c) \geq 1$.  Then, since $\check{\phi}(z)$ can be written as a convergent power series with non-negative coefficients, we have that $\check{\phi}(z) \geq z$ for $z \in (c, 1]$.  Hence either $\check{\phi}(z) = z$ on some subset of $(c, 1]$ or $\check{\phi}(z) > z$ for $z \in (c, 1]$.  In the former case, analytic continuation implies that $\check{\phi}(z) = z$ on $[0, 1]$, and in the latter case, $\check{\phi}(1) > 1$.   Thus, in either case we reach a contradiction and thus we can conclude that $\check{\phi}'(c) < 1$.  Therefore,it follows that $c$ is the unique fixed point attractor of $\check{\phi}(z)$.  

Lastly, since $c \in (0, 1)$, we can conclude that the infinite depth NTK solves the equilibrium equation corresponding to the recursive formula for the NTK in Eq.~\eqref{eq: NTK Recursive Formula}. Namely, for any $z \in (0, 1)$ and $K^{*}(z) := \lim_{L \to \infty} K^{(L)}(z)$: 
\begin{align*}
    K^*(z) = K^*(z) \check{\phi}'(c) + c \implies K^*(z) = \frac{c}{1 - \check{\phi}'(c)}.
\end{align*}
Hence, for any $z \in (0, 1)$, it holds that $\lim_{L \to \infty} K^{(L)}(z) = \frac{c}{1 - \check{\phi}'(c)}$.  Lastly, letting $a = \check{\phi}'(1)$, for $z = 1$, we have that
\begin{align*}
    K^{(L)}(1) = \frac{a^L - 1}{a - 1}, 
\end{align*}
and so $\lim_{L \to \infty} K^{(L)}(1) = \infty$.  Thus, $\lim_{L \to \infty}K^{(L)}(x, \tilde{x})$ satisfies the conditions of Proposition \ref{prop: NTK Majority Vote Condition}, which concludes the proof of the lemma.  
\end{proof}

We next show that if $\check{\phi}$ falls under case 3 
with $\check{\phi}'(1) < 1$, then under ridge regularization, the corresponding infinitely wide and deep classifier also implements majority vote classification. 

\begin{lemma}
\label{lemma: ordered phase majority vote}
Let $m_{n, \lambda}^{(L)}$ denote the ridge-regularized kernel machine with regularization term $\lambda$ and with the NTK of a fully connected network with $L$ hidden layers and activation function $\phi$ on $n$ training points.  If $\check{\phi}$ satisfies: 
\begin{enumerate}
    \item[1)] $\check{\phi}(0) > 0$, $\check{\phi}(1) = 1$,
    \item[2)] $\check{\phi}'(1) < 1$,
\end{enumerate}
then $\lim\limits_{\lambda \to 0^+} \lim\limits_{L \to \infty} m_{n, \lambda}^{(L)}(x)$ is the majority vote classifier. 
\end{lemma}

\begin{proof}
The proof follows that of Proposition \ref{prop: NTK Majority Vote Condition}.  Since $\check{\phi}'(1) < 1$, $z = 1$ is the unique fixed point attractor of $\check{\phi}$. Then as in the proof of Lemma \ref{lemma: chaotic phase majority vote}, for all $x, \tilde{x} \in \mathcal{S}_+^d$, it holds that
$$\lim\limits_{L \to \infty} K^{(L)}(x, \tilde{x}) = \frac{1}{1 - \check{\phi}'(1)}.$$  
Letting $c = \frac{1}{1 - \check{\phi}'(1)}$, we obtain that 
\begin{align*}
    \lim_{L \to \infty} m_{n, \lambda}^{(L)}(x) &= \sign \left(y (K_n^{(L)} + \lambda I)^{-1} K^{(L)}(X, x) \right) \\
    &= \sign \left( y \left[\lim_{L \to \infty} (K_n^{(L)} + \lambda I)^{-1} \right] \left[ \lim_{L \to \infty} K^{(L)}(X, x) \right]  \right) \\
    &= \sign \left( y \left[ \frac{1}{\lambda}I - \frac{c}{\lambda (\lambda + cn)} J \right]  c\mathbf{1}  \right) \\
    &= \sign \left( \frac{c}{\lambda + cn} \sum_{i=1}^{n} y^{(i)}  \right),
\end{align*}
where $J \in \mathbb{R}^{n \times n}$ is the all-ones matrix, $\mathbf{1} \in \mathbb{R}^n$ is the all-ones vector, and the third equality follows from the Sherman-Morrison formula. Hence, Proposition \ref{prop: NTK Majority Vote Condition} implies that $\lim\limits_{\lambda \to 0^+} \lim\limits_{L \to \infty} m_{n, \lambda}^{(L)}(x)$ is again the majority vote classifier, thereby completing the proof.
\end{proof}

\end{document}